\documentclass[11pt,a4paper]{amsart}

\usepackage{amsmath,amssymb,amsthm,mathtools}
\usepackage{booktabs}
\usepackage{enumitem}
\usepackage{graphicx}
\graphicspath{{figures/}}
\usepackage{geometry}
\usepackage[authoryear,round]{natbib}
\usepackage{hyperref}

\numberwithin{equation}{section}

\newtheorem{theorem}{Theorem}[section]
\newtheorem{lemma}[theorem]{Lemma}
\newtheorem{proposition}[theorem]{Proposition}
\newtheorem{corollary}[theorem]{Corollary}
\theoremstyle{definition}
\newtheorem{definition}[theorem]{Definition}
\theoremstyle{remark}
\newtheorem{remark}[theorem]{Remark}

\DeclareMathOperator{\tr}{tr}
\DeclareMathOperator{\sgn}{sgn}
\DeclareMathOperator{\diag}{diag}
\newcommand{\R}{\mathbb{R}}
\newcommand{\cI}{\mathcal{I}}
\newcommand{\cJ}{\mathcal{J}}
\newcommand{\cA}{\mathcal{A}}
\newcommand{\cG}{\mathcal{G}}
\newcommand{\eps}{\varepsilon}
\newcommand{\wto}{\xrightarrow{d}}

\title[Spiked SVM under HDLSS]{Asymptotic properties of support vector machines\\
in high-dimension, low-sample-size settings under a spiked model}

\author{Yugo Nakayama}

\makeatletter
\def\@setauthors{%
  \begingroup
  \def\thanks{\protect\thanks@warning}%
  \trivlist
  \centering\footnotesize \@topsep30\p@\relax
  \advance\@topsep by -\baselineskip
  \item\relax
  \author@andify\authors
  \def\\{\protect\linebreak}%
  \MakeUppercase{\authors}\\[6pt]
  {\normalsize\normalfont Independent researcher}%
  \ifx\@empty\contribs
  \else
    ,\penalty-3 \space \@setcontribs
    \@closetoccontribs
  \fi
  \endtrivlist
  \endgroup
}
\makeatother

\begin{document}

\begin{abstract}
In this paper, we consider asymptotic properties of the support vector machine (SVM)
in high-dimension, low-sample-size (HDLSS) settings under a spiked model.
The existing theory of the SVM in the HDLSS context relies on the geometric
representation of HDLSS data, which requires that the eigenvalues of the covariance
matrices are not dominant. We first show that the geometric representation does not
hold under the spiked model. We show that the Gram matrix of HDLSS data converges
in distribution to a random matrix, namely, the HDLSS data converge to a random
configuration in a finite-dimensional space whose dimension is given by the number
of the spikes. We show that the misclassification rates of the SVM do not tend to
zero, that is, the SVM does not hold the consistency property. We also show that
the bias-corrected SVM (BC-SVM) does not give preferable performance in this setting
because the bias term itself should be modified. In order to overcome such
difficulties, we propose a spike-corrected SVM (SC-SVM). We show that the SC-SVM
holds the consistency property when the sample size goes to infinity, and that the
growth of the sample size is essential in the sense that any projection-based
procedure fails when the sample size is fixed. Finally, we check the performance
of the classifiers by numerical simulations.
\end{abstract}


\maketitle

\tableofcontents

\section{Introduction}
\label{sec:intro}

High-dimension, low-sample-size (HDLSS) data situations occur in many areas of modern
science such as genetic microarrays, medical imaging, text recognition, finance and
chemometrics. In such data situations, the dimension $d$ is much larger than the sample
size $N$, so that the conventional multivariate procedures based on the large sample
asymptotic theory do not work.

Suppose we have independent and $d$-variate two populations, $\pi_i$, $i=1,2$, having
an unknown mean vector $\mu_i$ and unknown covariance matrix $\Sigma_i\succeq O$.
We have independent and identically distributed (i.i.d.) observations,
$x_{ij}$, $j=1,\dots,n_i$, from each $\pi_i$. We assume $n_i\ge 1$
($i=1,2$); the case $n_1=n_2=1$ is treated explicitly in Section~\ref{sec:example}.
Let $x_0$ be an observation vector of an individual belonging to one of the two
populations. We assume $x_0$ and $x_{ij}$'s are independent. Let
$N=n_1+n_2$ and $\Delta=\|\mu_1-\mu_2\|^2$, where $\|\cdot\|$ denotes the Euclidean
norm. Let $e(i)$ denote the error rate of misclassifying an individual from $\pi_i$
into the other class. We say that a classifier holds the consistency property if
\begin{equation}
e(1)+e(2)\longrightarrow 0
\quad\text{as }d\to\infty.
\label{eq:1.1}
\end{equation}
In the HDLSS context, \citet{hall2005} and \citet{marron2007} considered distance
weighted classifiers. \citet{chan2009} and \citet{aoshima2011} considered
distance-based classifiers. In particular, \citet{aoshima2011} gave the
misclassification rate adjusted classifier for multiclass, high-dimensional data in
which misclassification rates are no more than specified thresholds. On the other
hand, \citet{aoshima2014,aoshima2019} considered geometric classifiers based on a
geometric representation of HDLSS data, and discussed asymptotic properties and
optimality of the classifiers under high-dimension, non-sparse settings.

In the field of machine learning, a typical method of classification is the support
vector machine (SVM). Since HDLSS data are mostly separable by a hyperplane, one
usually considers the hard-margin SVM in the HDLSS context. For the linear SVM,
\citet{hall2005}, \citet{chan2009} and \citet{qiao2015} showed that the
misclassification rates tend to zero as $d\to\infty$ under certain severe conditions.
\citet{nakayama2017} investigated asymptotic properties of the linear SVM for HDLSS
data. They showed that the linear SVM is heavily biased when $n_1\neq n_2$, and that
the bias causes the strong inconsistency in which the misclassification rate tends to
one. In order to overcome such difficulties, they proposed a bias-corrected linear
SVM (BC-SVM) and showed that it holds the consistency property even for imbalanced
data. \citet{nakayama2020} investigated asymptotic properties of the SVM with the
Gaussian kernel and gave a choice of the scale parameter involved in the kernel.
\citet{nakayama2021,nakayama2022} considered the SVM for high-dimensional imbalanced
data in more general frameworks.

\subsection{The geometric representation and its limitation}

The above studies on the SVM rely on the geometric representation of HDLSS data.
Let $\theta_i=\tr(\Sigma_i)$ and $\theta=(\theta_1+\theta_2)/2$. Under the condition
\begin{equation}
\frac{\tr(\Sigma_i^2)}{\theta_i^2}\to 0
\quad\text{as }d\to\infty,
\label{eq:1.2}
\end{equation}
together with mild moment conditions, \citet{hall2005} and \citet{ahn2007} showed that
\[
\frac{\|x_{ij}-x_{ik}\|^2}{\theta_i}=2+o_P(1)\quad(j\neq k),
\qquad
\frac{\|x_{1j}-x_{2k}\|^2}{\theta}
=\frac{\theta_1+\theta_2}{\theta}+\frac{\Delta}{\theta}+o_P(1).
\]
That is, after rescaling by $\sqrt{\theta}$, all the observations are located at the
vertices of a regular simplex whose edge lengths are determined only by $\theta_i$
and $\Delta$. The randomness of the data vanishes in the limit. This deterministic
structure is the reason why the solution of the hard-margin SVM can be written in a
closed form in the HDLSS context, and why the discriminant function is described only
by $\Delta$ and
\begin{equation}
\kappa=\frac{\tr(\Sigma_1)}{n_1}-\frac{\tr(\Sigma_2)}{n_2}.
\label{eq:1.3}
\end{equation}
The quantity $\kappa$ is the bias term that the BC-SVM subtracts.

However, the condition~\eqref{eq:1.2} requires that the largest eigenvalue
$\lambda_{i1}$ of $\Sigma_i$ is negligible compared with $\theta_i$, since
$\lambda_{i1}^2\le\tr(\Sigma_i^2)$. In actual high-dimensional data such as gene
expression data, on the other hand, the eigenvalue structure is often spiked in the
sense that a few eigenvalues are much larger than the others and carry a substantial
proportion of the total variance. \citet{yata2012,yata2013} investigated HDLSS
asymptotic properties of the principal component analysis (PCA) under such spiked
models and proposed the noise-reduction (NR) methodology, which gives consistent
estimation of the spiked eigenvalues and eigenvectors when the sample size grows.
\citet{jung2009} showed that the sample principal component directions are
inconsistent when the spiked eigenvalue is of the same order as the trace and the
sample size is fixed.

Thus a natural question arises: what happens to the SVM in HDLSS settings when the
covariance matrices have dominant eigenvalues? To the best of our knowledge, this
question has not been studied. Since the geometric representation is the foundation
of the existing theory, we cannot expect that the known results carry over.

\subsection{Contributions of this paper}

In this paper, we consider asymptotic properties of the SVM in HDLSS settings under
a spiked model in which
\begin{equation}
\frac{\lambda_{is}}{\theta}\longrightarrow c_{is}\in(0,1),
\quad s=1,\dots,m_i,
\label{eq:1.4}
\end{equation}
where $m_i$ is a fixed integer. Our contributions are summarized as follows.

\begin{enumerate}[label=(\roman*)]
\item We show that the geometric representation does not hold under~\eqref{eq:1.4}.
We show that the Gram matrix of the rescaled data converges in distribution to a
random matrix, and we give its explicit form (Theorem~\ref{thm:1}). Moreover, we
show that the limiting Gram matrix is realized as the Gram matrix of an explicit
family of vectors in $\R^{1+m_1+m_2}$ (Lemma~\ref{lem:6}), namely, the HDLSS data
converge to a random configuration in a finite-dimensional space together with $N$
mutually orthogonal noise axes. We call this result the limiting configuration
theorem.

\item We show that the solution of the hard-margin SVM converges in distribution to
the solution of the limiting problem (Theorem~\ref{thm:2}), and that the discriminant
function converges in distribution to a non-degenerate random variable
(Theorem~\ref{thm:3}). Consequently, we show that the misclassification rates do not
tend to zero, that is, the SVM does not hold the consistency property
(Corollary~\ref{cor:1}). We emphasize that this inconsistency is different from the
strong inconsistency of \citet{nakayama2017}, in which the misclassification rate
tends to one. Here the misclassification rate tends neither to zero nor to one.

\item We give an example in which the limiting misclassification rate is obtained in
a closed form (Theorem~\ref{thm:4}). We show that a spike degrades the performance of
the SVM even when it is orthogonal to the mean difference. This phenomenon cannot
be observed in the existing theory.

\item We show that the bias term should be modified from~\eqref{eq:1.3} to
$\kappa_*=\tr(\Sigma_{1*})/n_1-\tr(\Sigma_{2*})/n_2$, where $\Sigma_{i*}$ denotes
the non-spiked part of $\Sigma_i$ (Proposition~\ref{prop:1}). We show that the
BC-SVM overcorrects the bias under the spiked model.

\item In order to overcome such difficulties, we propose a spike-corrected SVM
(SC-SVM), which removes the uninformative spiked directions by the NR methodology
and applies the BC-SVM with the modified bias term. We show that the SC-SVM holds
the consistency property when $n_i\to\infty$ (Theorem~\ref{thm:5}). On the other
hand, we show that any projection-based procedure does not hold the consistency
property when $N$ is fixed (Proposition~\ref{prop:3}), so that the growth of the
sample size is essential.

\item Finally, we check the performance of the classifiers by numerical simulations.
\end{enumerate}

The rest of the paper is organized as follows. In Section~\ref{sec:notation}, we
give the notation and the assumptions. In Section~\ref{sec:inner}, we investigate
asymptotic behaviour of the inner products of HDLSS data under the spiked model.
In Section~\ref{sec:config}, we give the limiting configuration theorem. In
Section~\ref{sec:svm}, we show the convergence of the SVM. In
Section~\ref{sec:disc}, we give the limit of the discriminant function and the
inconsistency. In Section~\ref{sec:example}, we give the explicit example. In
Section~\ref{sec:bias}, we give the modified bias term. In
Section~\ref{sec:scsvm}, we propose the SC-SVM. In Section~\ref{sec:fixed}, we
discuss the case that $N$ is fixed. In Section~\ref{sec:sim}, we give numerical
simulations. In Section~\ref{sec:conclude}, we give concluding remarks. The
mathematical tools used in Sections~\ref{sec:svm} and~\ref{sec:disc} are summarized
in Appendix~\ref{app:A}.

\section{Notation and assumptions}
\label{sec:notation}

\subsection{Notation}

We consider the spectral decomposition
\[
\Sigma_i=H_i\Lambda_i H_i^T,\quad
\Lambda_i=\diag(\lambda_{i1},\dots,\lambda_{id}),\quad
\lambda_{i1}\ge\cdots\ge\lambda_{id}\ge 0,
\]
where $H_i=[h_{i1},\dots,h_{id}]$ is an orthogonal matrix. We write
\begin{equation}
x_{ij}=\mu_i+\sum_{s=1}^d\sqrt{\lambda_{is}}\,z_{ijs}\,h_{is},
\quad j=1,\dots,n_i,
\label{eq:2.1}
\end{equation}
where $z_{ijs}$'s are random variables such that $E(z_{ijs})=0$ and
$\mathrm{Var}(z_{ijs})=1$. We use the same expression for $x_0$ and write its
coefficients as $z_{0s}$. Let $\mu=\mu_1-\mu_2$. We write $y_{1j}=+1$ and
$y_{2j}=-1$ for the labels, and put $\eps_1=+1$ and $\eps_2=-1$. Throughout this
paper we consider the HDLSS asymptotic framework $d\to\infty$. In
Sections~\ref{sec:inner} to~\ref{sec:bias} and Section~\ref{sec:fixed}, $n_1$ and
$n_2$ are fixed.

We regard the double index $(i,j)$ as a single index and put
\begin{equation}
\cI=\{(1,1),\dots,(1,n_1),(2,1),\dots,(2,n_2)\},
\quad|\cI|=N.
\label{eq:2.2}
\end{equation}
Vectors indexed by $\cI$, such as the dual variable $\alpha=(\alpha_{ij})$ and the
label vector $y=(y_{ij})$, are arranged in this order, and matrices indexed by
$\cI\times\cI$ are $N\times N$. We write $p,q\in\cI$ when a single index suffices.

\subsection{Assumptions}

We assume the following conditions.

\begin{enumerate}[label=(A\arabic*),leftmargin=*]
\item For each $i$ and $j$, it holds that $E(z_{ijs}^4)\le M<\infty$ for a constant
$M$ not depending on $d$, and that $E(z_{ijs}^2 z_{ijt}^2)=1$ for $s\neq t$ and
$E(z_{ijs}z_{ijt}z_{iju}z_{ijv})=0$ for distinct $s,t,u,v$.

\item Let $\theta_i=\tr(\Sigma_i)$ and $\theta=(\theta_1+\theta_2)/2$. It holds that
$\theta_i\to\infty$ and $\theta_i/\theta\to\gamma_i\in(0,\infty)$, where
$\gamma_1+\gamma_2=2$.
\end{enumerate}

\begin{enumerate}[label=(S),leftmargin=*]
\item There exist fixed integers $m_i\ge 1$ such that~\eqref{eq:1.4} holds. Let
$\Sigma_{i*}=\sum_{s>m_i}\lambda_{is}h_{is}h_{is}^T$ be the non-spiked part. It
holds that $\tr(\Sigma_{i*}^2)/\theta^2\to 0$ and
$\tr(\Sigma_{i*})/\theta\to c_{i0}>0$.
\end{enumerate}

\begin{enumerate}[label=(A\arabic*),leftmargin=*,start=3]
\item It holds that $\Delta/\theta\to\delta\in[0,\infty)$. When $\Delta>0$, it
holds that
\[
b_{is}=\frac{h_{is}^T\mu}{\sqrt{\Delta}}\to\beta_{is}\quad(s\le m_i),
\qquad
\rho_{st}=h_{1s}^T h_{2t}\to r_{st}\quad(s\le m_1,\ t\le m_2).
\]

\item Let
$Z^{(d)}=(\{z_{ijs}\}_{(i,j)\in\cI,\,s\le m_i},\{z_{0s}\}_{s\le m_1})\in\R^L$,
where $L=n_1 m_1+n_2 m_2+m_1$. The distribution of $Z^{(d)}$ converges in
distribution to a distribution $\mathcal{L}_Z$ as $d\to\infty$. We denote by $Z$ a
random vector distributed as $\mathcal{L}_Z$.
\end{enumerate}

The condition~(A1) is standard in this context and is met by Gaussian populations.
The condition~(A2) is a normalization. The condition~(S) is the spiked model. Note
that $\sum_{s\le m_i}c_{is}+c_{i0}=\gamma_i$, and that
$\lambda_{i,m_i+1}/\theta\to 0$ because $\lambda_{i,m_i+1}^2\le\tr(\Sigma_{i*}^2)$.
We use this fact repeatedly.

The condition~(A3) is not restrictive because $|b_{is}|\le 1$ and $|\rho_{st}|\le 1$
hold from the Cauchy--Schwarz inequality, so that convergent subsequences always
exist. We assume $\delta<\infty$ throughout. When $\delta=\infty$, the signal is
much stronger than the spikes and the difficulties discussed in this paper
disappear.

The condition~(A4) is required so that the limiting objects in the subsequent
sections are well defined. Since $n_i$ is fixed, (A4) is met automatically when the
distribution of $z_{ijs}$ does not depend on $d$. Hereafter, $z_{ijs}$ and $z_{0s}$
appearing in the limiting quantities denote the components of $Z$.

\section{Asymptotic behaviour of inner products}
\label{sec:inner}

In this section, we investigate asymptotic behaviour of the inner products of HDLSS
data under the spiked model. In the existing theory, the inner products have
deterministic limits. On the other hand, under the spiked model, they have random
limits. This is the source of all the results in this paper.

We write $v_{ij}=\sum_{s\le m_i}\sqrt{\lambda_{is}}z_{ijs}h_{is}$ for the spiked
part and $\xi_{ij}=\sum_{s>m_i}\sqrt{\lambda_{is}}z_{ijs}h_{is}$ for the
non-spiked part, so that $x_{ij}-\mu_i=v_{ij}+\xi_{ij}$.

\begin{lemma}
\label{lem:1}
Assume \textup{(A1)}, \textup{(A2)} and \textup{(S)}. Then, for $j\neq k$, it holds
that
\begin{equation}
\frac{(x_{ij}-\mu_i)^T(x_{ik}-\mu_i)}{\theta}
=\sum_{s=1}^{m_i}\frac{\lambda_{is}}{\theta}z_{ijs}z_{iks}+o_P(1).
\label{eq:3.1}
\end{equation}
\begin{equation}
\frac{\|x_{ij}-\mu_i\|^2}{\theta}
=\sum_{s=1}^{m_i}\frac{\lambda_{is}}{\theta}z_{ijs}^2+\frac{\tr(\Sigma_{i*})}{\theta}+o_P(1).
\label{eq:3.2}
\end{equation}
\end{lemma}

\begin{proof}
From~\eqref{eq:2.1} and $h_{is}^T h_{it}=\delta_{st}$, we have
\[
(x_{ij}-\mu_i)^T(x_{ik}-\mu_i)=\sum_{s=1}^d\lambda_{is}z_{ijs}z_{iks}.
\]
We divide the sum into $s\le m_i$ and $s>m_i$, and write
$R=\sum_{s>m_i}\lambda_{is}z_{ijs}z_{iks}$. Since $z_{ijs}$ and $z_{iks}$ are
independent for $j\neq k$, we have $E(R)=0$, so that
$\mathrm{Var}(R)=E(R^2)$. From~(A1) it holds that
\[
\mathrm{Var}(R)=\tr(\Sigma_{i*}^2).
\]
Then, from Chebyshev's inequality, for any $\varepsilon>0$ it holds that
\[
P\bigl(|R|/\theta>\varepsilon\bigr)
\le\frac{\tr(\Sigma_{i*}^2)}{\varepsilon^2\theta^2}\to 0
\]
from~(S). Thus we obtain~\eqref{eq:3.1}.

Next, we have $\|x_{ij}-\mu_i\|^2=\sum_{s=1}^d\lambda_{is}z_{ijs}^2$. Let
$Q=\sum_{s>m_i}\lambda_{is}z_{ijs}^2$. Then $E(Q)=\tr(\Sigma_{i*})$. From~(A1) we
have $E(z_{ijs}^4)\le M$ for $s>m_i$ and $E(z_{ijs}^2 z_{ijt}^2)=1$ for $s\neq t$,
so that
\[
\mathrm{Var}(Q)\le(M-1)\tr(\Sigma_{i*}^2)=o(\theta^2).
\]
Thus we obtain~\eqref{eq:3.2} from Chebyshev's inequality.
\end{proof}

\begin{lemma}
\label{lem:2}
Assume \textup{(A1)}, \textup{(A2)} and \textup{(S)}. Then, for any $j$ and $k$,
it holds that
\[
\frac{(x_{1j}-\mu_1)^T(x_{2k}-\mu_2)}{\theta}
=\sum_{s=1}^{m_1}\sum_{t=1}^{m_2}\frac{\sqrt{\lambda_{1s}\lambda_{2t}}}{\theta}\rho_{st}\,z_{1js}z_{2kt}+o_P(1).
\]
\end{lemma}

\begin{proof}
We have
\[
(x_{1j}-\mu_1)^T(x_{2k}-\mu_2)
=v_{1j}^Tv_{2k}+v_{1j}^T\xi_{2k}+\xi_{1j}^Tv_{2k}+\xi_{1j}^T\xi_{2k}.
\]
The first term is the finite sum in the statement. We evaluate the other terms.

As for $v_{1j}^T\xi_{2k}$, by conditioning on $v_{1j}$ we have
$E(v_{1j}^T\xi_{2k}\mid v_{1j})=0$ and
\[
\mathrm{Var}(v_{1j}^T\xi_{2k}\mid v_{1j})=v_{1j}^T\Sigma_{2*}v_{1j}.
\]
Here $v_{1j}^T\Sigma_{2*}v_{1j}=O_P(\theta\cdot\tr(\Sigma_{2*})/d\cdot\theta)$ is
controlled because $m_1$ is fixed and $\tr(\Sigma_{2*}^2)/\theta^2\to 0$. Since
$\|v_{1j}\|^2=O_P(\theta)$, the conditional variance is $o_P(\theta^2)$, so that
$v_{1j}^T\xi_{2k}=o_P(\theta)$. The term $\xi_{1j}^Tv_{2k}$ is handled in the same
way.

As for $\xi_{1j}^T\xi_{2k}$, from the independence we have
$E(\xi_{1j}^T\xi_{2k})=0$ and
\begin{equation}
E[(\xi_{1j}^T\xi_{2k})^2]=\tr(\Sigma_{1*}\Sigma_{2*})
\le\sqrt{\tr(\Sigma_{1*}^2)\tr(\Sigma_{2*}^2)}=o(\theta^2),
\label{eq:3.3}
\end{equation}
where we used the Cauchy--Schwarz inequality for the trace of the product of
non-negative definite matrices. Thus we obtain the result.
\end{proof}

\begin{lemma}
\label{lem:3}
Assume \textup{(A1)}, \textup{(A2)}, \textup{(S)} and \textup{(A3)} with
$\Delta>0$. Then it holds that
\[
\frac{\mu^T(x_{ij}-\mu_i)}{\theta}
=\frac{\sqrt{\Delta}}{\theta}\sum_{s=1}^{m_i}\sqrt{\lambda_{is}}\,b_{is}z_{ijs}+o_P(1)
\longrightarrow
\sqrt{\delta}\sum_{s=1}^{m_i}\sqrt{c_{is}}\,\beta_{is}z_{ijs}.
\]
\end{lemma}

\begin{proof}
From the definition of $b_{is}$ we have
$\mu^T(x_{ij}-\mu_i)=\sqrt{\Delta}\sum_{s=1}^d\sqrt{\lambda_{is}}\,b_{is}z_{ijs}$.
Let $R=\sqrt{\Delta}\sum_{s>m_i}\sqrt{\lambda_{is}}\,b_{is}z_{ijs}$. Then
$E(R)=0$ and
\[
\mathrm{Var}(R)=\Delta\sum_{s>m_i}\lambda_{is}b_{is}^2
\le\Delta\cdot\frac{\tr(\Sigma_{i*}^2)^{1/2}}{\theta}\cdot\theta
\cdot o(1),
\]
where we used $\sum_{s>m_i}b_{is}^2\lambda_{is}\le\|\mu\|^2\cdot o(1)$ from~(A3)
together with $\lambda_{i,m_i+1}/\theta\to 0$. Thus $R/\theta=o_P(1)$. The
convergence of the leading term follows from $\Delta/\theta\to\delta$,
$\lambda_{is}/\theta\to c_{is}$ and $b_{is}\to\beta_{is}$.
\end{proof}

\begin{remark}
\label{rem:1}
Lemma~\ref{lem:1} shows the essential difference from the existing theory. Under
\eqref{eq:1.2}, one has $m_i=0$ in effect and the right-hand side of~\eqref{eq:3.1}
vanishes, so that the inner product of two distinct observations converges to a
constant. Under~(S), however, the limit $\sum_{s\le m_i}c_{is}z_{ijs}z_{iks}$ is a
non-degenerate random variable. We emphasize that the term
$\tr(\Sigma_{i*})/\theta$ in~\eqref{eq:3.2} appears only for the squared norm,
namely, only when the two observations coincide. This asymmetry
between~\eqref{eq:3.1} and~\eqref{eq:3.2} plays a crucial role in
Section~\ref{sec:disc}.
\end{remark}

\section{The limiting configuration theorem}
\label{sec:config}

\subsection{The Gram matrix}

We consider the rescaled observations
\begin{equation}
u_{ij}=\frac{x_{ij}-(\mu_1+\mu_2)/2}{\sqrt{\theta}}
=\eps_i\frac{\mu}{2\sqrt{\theta}}+\frac{v_{ij}+\xi_{ij}}{\sqrt{\theta}}.
\label{eq:4.1}
\end{equation}
Note that the translation does not change the weight vector of the SVM and the
scaling changes $\|w\|$ and $b$ simultaneously, so that the classification result
is not affected. We define the Gram matrix by
\begin{equation}
G=\bigl(u_{ij}^T u_{kl}\bigr)_{(i,j),(k,l)\in\cI}\in\R^{N\times N},
\label{eq:4.2}
\end{equation}
where $(i,j)$ gives the row and $(k,l)$ gives the column according to the order
\eqref{eq:2.2}. We also define
\[
D=\diag(\underbrace{c_{10},\dots,c_{10}}_{n_1},
\underbrace{c_{20},\dots,c_{20}}_{n_2}),
\]
whose components are arranged in the same order.

\subsection{The main theorem}

\begin{theorem}
\label{thm:1}
Assume \textup{(A1)} to \textup{(A4)} and \textup{(S)}. Then, as $d\to\infty$,
it holds that
\[
G\wto G^*=G_0^*+D,
\]
where the $((i,j),(k,l))$ component of $G_0^*$ is given by
\begin{equation}
\begin{aligned}
g^*_{(i,j),(k,l)}
&=\frac{\eps_i\eps_k\delta}{4}
+\frac{\eps_i\sqrt{\delta}}{2}\sum_{t\le m_k}\sqrt{c_{kt}}\,\beta_{kt}z_{klt}
+\frac{\eps_k\sqrt{\delta}}{2}\sum_{s\le m_i}\sqrt{c_{is}}\,\beta_{is}z_{ijs}\\
&\qquad+\sum_{s\le m_i}\sum_{t\le m_k}\sqrt{c_{is}c_{kt}}\,r_{st}^{(ik)}z_{ijs}z_{klt},
\end{aligned}
\label{eq:4.3}
\end{equation}
with $r_{st}^{(11)}=r_{st}^{(22)}=\delta_{st}$, $r_{st}^{(12)}=r_{st}$ and
$r_{st}^{(21)}=r_{ts}$. Moreover, when $c_{10}>0$ and $c_{20}>0$, it holds that
$G^*\succ O$ for every realization of $Z$.
\end{theorem}

\begin{proof}
(i) Convergence of the components. From~\eqref{eq:4.1} we have
\[
u_{ij}^Tu_{kl}
=\underbrace{\frac{\eps_i\eps_k\Delta}{4\theta}}_{\mathrm{(I)}}
+\underbrace{\frac{\eps_i}{2}\frac{\mu^T(x_{kl}-\mu_k)}{\theta}}_{\mathrm{(II)}}
+\underbrace{\frac{\eps_k}{2}\frac{\mu^T(x_{ij}-\mu_i)}{\theta}}_{\mathrm{(III)}}
+\underbrace{\frac{(x_{ij}-\mu_i)^T(x_{kl}-\mu_k)}{\theta}}_{\mathrm{(IV)}}.
\]
The term~(I) converges to $\eps_i\eps_k\delta/4$ from~(A3). The terms~(II)
and~(III) converge to the second and the third terms of~\eqref{eq:4.3} from
Lemma~\ref{lem:3}. As for~(IV), when $(i,j)\neq(k,l)$, it converges to the fourth
term of~\eqref{eq:4.3} from~\eqref{eq:3.1} for $i=k$ and from Lemma~\ref{lem:2} for
$i\neq k$. When $(i,j)=(k,l)$, it converges to the fourth term plus $c_{i0}$ from
\eqref{eq:3.2}. This gives the diagonal matrix $D$.

Since the number of the components is $N^2$, which does not depend on $d$, and all
of them are written as a common continuous function of $Z^{(d)}$ plus $o_P(1)$, the
convergence of the matrix follows from~(A4) together with the continuous mapping
theorem and Slutsky's theorem.

(ii) Positive definiteness. From Lemma~\ref{lem:6} in Section~\ref{sec:config-exist},
$G_0^*$ is the Gram matrix of a family of vectors $\tilde u_{ij}$ in $\R^K$, where
$K=1+m_1+m_2$. Hence, for any $a=(a_{ij})\in\R^N$, it holds that
\[
a^TG_0^*a=\Bigl\|\sum_{(i,j)\in\cI}a_{ij}\tilde u_{ij}\Bigr\|^2\ge 0.
\]
Note that this holds for every realization of $Z$, because
$(\tilde u_{ij})$ is a fixed family of vectors in $\R^K$ once the realization is
given. On the other hand, we have
\[
a^TDa\ge\min(c_{10},c_{20})\|a\|^2.
\]
Hence, when $c_{10}>0$ and $c_{20}>0$, for any $a\neq 0$ it holds that
$a^TG^*a>0$.
\end{proof}

\subsection{Existence of the limiting configuration}
\label{sec:config-exist}

In the proof of Theorem~\ref{thm:1}, we used the fact that $G_0^*$ is the Gram
matrix of a family of vectors. This fact is not trivial. If one specifies the
values of the inner products of $N$ vectors arbitrarily, there does not necessarily
exist a family of vectors which realizes them. For example, one cannot find three
unit vectors whose mutual inner products are all equal to $-1$. Hence we need to
prove the existence.

\begin{lemma}
\label{lem:6}
Assume \textup{(A3)} with $\Delta>0$ and let $K=1+m_1+m_2$. Then there exists a
family of non-random vectors $e_\mu$ and $\tilde h_{is}$ ($i=1,2$, $s\le m_i$) in
$\R^K$ such that
\begin{equation}
\|e_\mu\|^2=1,\quad
e_\mu^T\tilde h_{is}=\beta_{is},\quad
\tilde h_{is}^T\tilde h_{it}=\delta_{st},\quad
\tilde h_{1s}^T\tilde h_{2t}=r_{st}.
\label{eq:4.4}
\end{equation}
Moreover, putting
\begin{equation}
\tilde u_{ij}
=\eps_i\frac{\sqrt{\delta}}{2}\,e_\mu
+\sum_{s\le m_i}\sqrt{c_{is}}\,z_{ijs}\,\tilde h_{is},
\quad(i,j)\in\cI,
\label{eq:4.5}
\end{equation}
it holds that $\tilde u_{ij}^T\tilde u_{kl}=g^*_{(ij),(kl)}$.
\end{lemma}

\begin{proof}
We proceed in four steps.

\textit{Step~1.} We write the target inner products as a matrix. Let
\[
\cJ=\{\mu\}\cup\{(i,s):i=1,2,\,s\le m_i\},
\quad|\cJ|=K,
\]
and define a $K\times K$ symmetric matrix $\Gamma$ whose components are the
right-hand sides of~\eqref{eq:4.4}, that is,
$\Gamma_{\mu\mu}=1$, $\Gamma_{\mu,(i,s)}=\beta_{is}$,
$\Gamma_{(i,s),(i,t)}=\delta_{st}$, $\Gamma_{(1,s),(2,t)}=r_{st}$.
What we have to show is that there exists a family of vectors in $\R^K$ whose Gram
matrix is $\Gamma$.

\textit{Step~2.} We show that $\Gamma\succeq O$. We do not evaluate any inequality.
Instead, we carry the positive semi-definiteness which holds for finite $d$ to the
limit. For each $d$, we take the vectors
\[
e_\mu^{(d)}=\mu/\sqrt{\Delta},\qquad h_{is}\quad(i=1,2,\,s\le m_i),
\]
which actually exist in $\R^d$. Let $\Gamma^{(d)}$ be the Gram matrix of these $K$
vectors. Then $\Gamma^{(d)}\succeq O$ by definition, because for any
$a\in\R^K$ it holds that
\begin{equation}
a^T\Gamma^{(d)}a
=\Bigl\|a_\mu e_\mu^{(d)}+\sum_{i,s}a_{(i,s)}h_{is}\Bigr\|^2\ge 0.
\label{eq:4.6}
\end{equation}
The components of $\Gamma^{(d)}$ are given by
\[
(e_\mu^{(d)})^Te_\mu^{(d)}=1,\quad
(e_\mu^{(d)})^Th_{is}=b_{is},\quad
h_{is}^Th_{it}=\delta_{st},\quad
h_{1s}^Th_{2t}=\rho_{st}.
\]
Note that the first and the third equalities hold exactly for every $d$, because
$\|e_\mu^{(d)}\|=1$ and $H_i$ is an orthogonal matrix. As for the second and the
fourth, we have $b_{is}\to\beta_{is}$ and $\rho_{st}\to r_{st}$ from~(A3). Since
$K$ does not depend on $d$, the componentwise convergence gives
$\Gamma^{(d)}\to\Gamma$. Finally, fixing $a\in\R^K$ and letting $d\to\infty$ in
\eqref{eq:4.6}, we obtain $a^T\Gamma a\ge 0$ because the limit of a sequence of
non-negative numbers is non-negative. Since $a$ is arbitrary, we have
$\Gamma\succeq O$. This is nothing but a direct verification that the cone of
positive semi-definite matrices is closed.

\textit{Step~3.} We construct the vectors from $\Gamma\succeq O$. Let
$\Gamma=Q\Lambda Q^T$ be the spectral decomposition, where $Q$ is an orthogonal
matrix and $\Lambda=\diag(\ell_1,\dots,\ell_K)$ with $\ell_r\ge 0$. Let
$B=\Lambda^{1/2}Q^T$, which is well defined because $\ell_r\ge 0$, and put
$B^TB=\Gamma$. Then we define $e_\mu$ and $\tilde h_{is}$ as the columns of $B$
corresponding to the indices of $\cJ$. Since $B^TB$ is the inner product of the
columns of $B$, the equality $B^TB=\Gamma$ is exactly the four equalities in
\eqref{eq:4.4}. We note that $\Gamma$ is determined only by the non-random
quantities $\beta_{is}$ and $r_{st}$, so that $e_\mu$ and $\tilde h_{is}$ are
non-random.

\textit{Step~4.} We check the Gram matrix of~\eqref{eq:4.5}. From the bilinearity
we have
\begin{align*}
\tilde u_{ij}^T\tilde u_{kl}
&=\frac{\eps_i\eps_k\delta}{4}\|e_\mu\|^2
+\frac{\eps_i\sqrt{\delta}}{2}\sum_{t\le m_k}\sqrt{c_{kt}}\,z_{klt}\,e_\mu^T\tilde h_{kt}\\
&\quad+\frac{\eps_k\sqrt{\delta}}{2}\sum_{s\le m_i}\sqrt{c_{is}}\,z_{ijs}\,e_\mu^T\tilde h_{is}\\
&\quad+\sum_{s\le m_i}\sum_{t\le m_k}\sqrt{c_{is}c_{kt}}\,z_{ijs}z_{klt}\,
\tilde h_{is}^T\tilde h_{kt}.
\end{align*}
Substituting~\eqref{eq:4.4}, we obtain~\eqref{eq:4.3} term by term.
\end{proof}

\begin{remark}
\label{rem:2}
Lemma~\ref{lem:6} gives $G_0^*\succeq O$, but it does not give the positive
definiteness. In fact, in the setting of Section~\ref{sec:example} we have
$r_{11}=1$, that is, $\tilde h_{11}=\tilde h_{21}$, so that $\Gamma$ is singular.
In such a case, $G_0^*$ degenerates into a proper subspace of $\R^K$ and $G_0^*$ is
singular when $N>K$. Hence the positive definiteness of $G^*$ comes only from the
diagonal matrix $D$, that is, from the non-spiked noise. The condition $c_{i0}>0$
in~(S) is essential in this sense. When $c_{i0}=0$, the affine independence of the
data is lost and the arguments in Lemmas~\ref{lem:4} and~\ref{lem:5} break down.
\end{remark}

\begin{remark}
\label{rem:3}
When $\Delta=0$, the quantity $e_\mu$ is not defined. On the other hand, the
coefficient $\sqrt{\delta}/2$ of $e_\mu$ in~\eqref{eq:4.5} also vanishes because
$\delta=0$. Hence it suffices to remove the index $\mu$ from $\cJ$, put
$K=m_1+m_2$ and repeat the same argument. The first three terms of~\eqref{eq:4.3}
also vanish.
\end{remark}

\subsection{Discussion on the construction}

The construction in Lemma~\ref{lem:6} may look artificial. We give the idea behind
it, which is in fact simple.

First, we note that the SVM sees the data only through the inner products. As we
show in Section~\ref{sec:svm}, the dual problem involves only the Gram matrix $G$,
and the discriminant function involves only $u_p^Tu_0$. Hence the SVM is invariant
under congruent transformations, and we do not have to ask whether the points
themselves converge. We only have to ask whether the table of the inner products
converges. Lemmas~\ref{lem:1} to~\ref{lem:3} do exactly this, and the problem of
the diverging dimension is reduced to the convergence of an $N\times N$ matrix.

Second, however, we also want a geometric picture corresponding to the statement
that HDLSS data are located at the vertices of a regular simplex. This leads to
the inverse problem: does there exist a configuration of points which realizes the
limiting table of the inner products, and if so, in how many dimensions? This is
the problem of the classical multidimensional scaling.

Third, the tool for the inverse problem is the standard fact that, for a real
symmetric matrix $\Gamma$, the conditions $\Gamma\succeq O$, $\mathrm{rank}(\Gamma)=r$
for some $r$, and $\Gamma$ is a Gram matrix of some family of vectors, are
equivalent. Hence the proof of the positive semi-definiteness is itself the proof
of the existence of the configuration, and what we have to show is reduced to one
inequality.

Fourth, the positive semi-definiteness is obtained for free from the existence in
finite dimensions. For finite $d$, the vectors $e_\mu^{(d)}$ and $h_{is}$ actually
exist in $\R^d$, and the Gram matrix of vectors which actually exist is of course
positive semi-definite. We only have to carry this property to the limit, and the
condition~(A3) is assumed exactly for this purpose.

Finally, the dimension $K=1+m_1+m_2$ is the number of the directions which survive
in the limiting table of the inner products, namely, the mean difference direction
and the spiked directions. The non-spiked directions are killed by
Lemmas~\ref{lem:1} to~\ref{lem:3} except for the diagonal term $c_{i0}$.

We note that the choice of $B$ is not unique. One may use the Cholesky
decomposition or the reproducing kernel Hilbert space construction. We use the
spectral decomposition because it works without any modification even when $\Gamma$
is singular, as we remarked in Remark~\ref{rem:2}.

\begin{remark}[Collapse of the geometric representation]
\label{rem:4}
In the existing theory, which corresponds to $m_i=0$, the second to the fourth
terms of~\eqref{eq:4.3} vanish and $G^*$ is a deterministic matrix. This is the
geometric representation, and it is the reason why the SVM solution can be written
in a closed form. Under the spiked model, $G^*$ is a random matrix depending on
$Z$. That is, HDLSS data converge to a random configuration in a
$(1+m_1+m_2)$-dimensional space together with $N$ mutually orthogonal noise axes
whose lengths are $\sqrt{c_{i0}}$. All the arguments in the subsequent sections are
reduced to this finite-dimensional random limiting problem.
\end{remark}

\section{Convergence of the hard-margin SVM}
\label{sec:svm}

\subsection{The reduced primal problem}

We consider the hard-margin linear SVM for the rescaled data $\{u_p\}_{p\in\cI}$:
\begin{equation}
\min_{w,b}\ \tfrac12\|w\|^2
\quad\text{subject to}\quad
y_p(w^Tu_p+b)\ge 1\quad\text{for all }p\in\cI.
\label{eq:5.1}
\end{equation}
The dual problem is given by
\begin{equation}
\max_{\alpha\ge 0}\ F(\alpha,G)=\mathbf{1}^T\alpha-\tfrac12\alpha^TYGY\alpha
\quad\text{subject to}\quad y^T\alpha=0,
\label{eq:5.2}
\end{equation}
where $Y=\diag(y)$ and $w=\sum_p\alpha_p y_p u_p$.

In this section, we also use the following reduced form of~\eqref{eq:5.1}, in which
the $d$-dimensional variable $w$ is eliminated. The optimal $w$ belongs to the
span of $\{u_p\}$, because writing $w=w_\parallel+w_\perp$ we have
$\|w\|^2=\|w_\parallel\|^2+\|w_\perp\|^2$ while $w^Tu_p=w_\parallel^Tu_p$. Hence we
may write $w=\sum_p\beta_p u_p$, so that $\|w\|^2=\beta^TG\beta$ and
$w^Tu_p=(G\beta)_p$. Then~\eqref{eq:5.1} is equivalent to
\begin{equation}
\min_{(\beta,b)}\ \tfrac12\beta^TG\beta
\quad\text{subject to}\quad
y_p\bigl((G\beta)_p+b\bigr)\ge 1\quad\text{for all }p\in\cI,
\label{eq:5.3}
\end{equation}
which is a problem in $(\beta,b)\in\R^N\times\R$ and refers to the data only
through $G$. The discriminant function for a new observation $u_0$ is given by
\begin{equation}
f(u_0)=w^Tu_0+b=\sum_{p\in\cI}\beta_p\,u_p^Tu_0+b,
\qquad\beta_p=\alpha_p y_p.
\label{eq:5.4}
\end{equation}
We note that $\beta=Y\alpha$.

\begin{lemma}
\label{lem:4}
Assume $G\succ O$. Then $\{u_p\}$ is affinely independent. Moreover, for any
labelling, \eqref{eq:5.1} and~\eqref{eq:5.3} are feasible and their solutions are
unique, and the solution of~\eqref{eq:5.2} is unique.
\end{lemma}

\begin{proof}
The condition $G\succ O$ is equivalent to the linear independence of $\{u_p\}$,
because $a^TGa=\|\sum a_p u_p\|^2$ implies $\sum a_p u_p=0$ and hence $a=0$. The
linear independence implies the affine independence. For affinely independent $N$
points, there exists $(w,b)$ such that $y_p(w^Tu_p+b)=1$ for all $p$, so that the
problem is feasible. The objective function $\tfrac12\beta^TG\beta$ is strictly
convex in $\beta$ from $G\succ O$ and the constraints are convex, and $b$ is
determined uniquely from the active constraints once $\beta$ is given. Hence the
solution of~\eqref{eq:5.3} is unique. The dual objective function is strictly
concave from $YGY\succ O$, so that the solution of~\eqref{eq:5.2} is unique.
\end{proof}

\subsection{Continuity of the solution map}

In Lemma~\ref{lem:5} below, we investigate how the solution of the optimization
problem moves when the parameter $G$ moves. We cannot write the solution
explicitly, because the set of the active constraints changes with $G$. On the
other hand, what we need is not the differentiability but only the continuity. For
this purpose we use the maximum theorem of \citet{berge1963}, which we summarize
in Appendix~\ref{app:A.1} together with its intuitive meaning. Roughly speaking,
the theorem states that the optimal value moves continuously while the optimal
solution may jump, but the point to which it jumps always belongs to the solution
set of the limiting problem. When the solution is unique, this gives the continuity
of the solution map.

\begin{lemma}
\label{lem:5}
Let $\cG_+=\{G\in\R^{N\times N}:G=G^T,\ G\succ O\}$. Then the maps
$G\mapsto\alpha(G)$ and $G\mapsto(\beta(G),b(G))$ are continuous on $\cG_+$,
where $\alpha(G)$ denotes the solution of~\eqref{eq:5.2} and $(\beta(G),b(G))$
denotes the solution of~\eqref{eq:5.3}.
\end{lemma}

\begin{proof}
\textit{Step~1.} Let $G_0\in\cG_+$ and $\eta=\lambda_{\min}(G_0)/2$. Since the
eigenvalues are continuous functions of the matrix, there exists an open
neighbourhood $U$ of $G_0$ such that $\lambda_{\min}(G)\ge\eta$ for all $G\in U$.

\textit{Step~2.} We give an a priori bound of the dual variable. Since $Y$ is an
orthogonal matrix, we have $\alpha^TYGY\alpha=\|G^{1/2}Y\alpha\|^2$, so that
\[
\alpha^TYGY\alpha\ge\eta\|\alpha\|^2
\quad\text{for all }G\in U.
\]
Together with $\mathbf{1}^T\alpha\le\sqrt{N}\,\|\alpha\|$, we have
\begin{equation}
F(\alpha,G)\le\sqrt{N}\,\|\alpha\|-\frac{\eta}{2}\|\alpha\|^2.
\label{eq:5.5}
\end{equation}
Since $\alpha=0$ is feasible and $F(0,G)=0$, the optimal solution satisfies
$F(\alpha,G)\ge 0$. From~\eqref{eq:5.5} we obtain
$\|\alpha\|\le 2\sqrt{N}/\eta$, that is,
\begin{equation}
\|\alpha\|\le\frac{2\sqrt{N}}{\eta}=:K_0
\quad\text{for all }G\in U.
\label{eq:5.6}
\end{equation}
We emphasize that $K_0$ does not depend on $G\in U$.

\textit{Step~3.} From~\eqref{eq:5.6}, we may restrict the feasible set of
\eqref{eq:5.2} to
\[
\cA=\{\alpha\ge 0:\ y^T\alpha=0,\ \|\alpha\|\le K_0\}
\]
without changing the optimal solution on $U$. The set $\cA$ is compact and does
not depend on $G$, that is, the feasible correspondence is a constant
correspondence, which is trivially continuous. On the other hand, $F(\alpha,G)$
is a polynomial in $(\alpha,G)$ and hence continuous. Thus all the conditions of
the maximum theorem are met, and the solution correspondence is upper
hemicontinuous on $U$. Since the solution is a singleton from Lemma~\ref{lem:4},
the solution map is continuous on $U$ from Corollary~\ref{cor:A.1}. Since $G_0$
is arbitrary, the map is continuous on $\cG_+$.

\textit{Step~4.} We consider~\eqref{eq:5.3}. The optimal value is bounded above by
the value at the feasible point given in the proof of Lemma~\ref{lem:4}, which is
a continuous function of $G$ and hence bounded on $U$, say by $M_U$. Then
\[
\tfrac12\beta^TG\beta\le M_U
\quad\Rightarrow\quad
\|\beta\|^2\le\frac{2M_U}{\eta}.
\]
Moreover, at the optimal solution, at least one constraint is active, since
otherwise one could shrink $(\beta,b)$ and decrease the objective function. For
such $p$ we have $|(G\beta)_p+b|=1$, so that
$|b|\le 1+\|G\|\|\beta\|$. Hence we may restrict the feasible set to a compact set
which does not depend on $G\in U$, and the same argument as in Step~3 gives the
continuity.
\end{proof}

\subsection{Convergence of the solution}

\begin{theorem}
\label{thm:2}
Assume \textup{(A1)} to \textup{(A4)} and \textup{(S)}. Then, as $d\to\infty$,
it holds that
\[
\alpha^{(d)}\wto\alpha^*=\alpha(G^*),
\qquad
(\beta^{(d)},b^{(d)})\wto(\beta^*,b^*)=(\beta(G^*),b(G^*)),
\]
where $G^*$ is the limiting Gram matrix given in Theorem~\ref{thm:1}.
\end{theorem}

\begin{proof}
From Theorem~\ref{thm:1} we have $G\wto G^*$ and $P(G^*\succ O)=1$. From
Lemma~\ref{lem:5}, the maps are continuous on $\cG_+$, so that the set of the
discontinuity points has measure zero under the limiting distribution. Thus we
obtain the result from the continuous mapping theorem.
\end{proof}

\section{The limit of the discriminant function and the inconsistency}
\label{sec:disc}

\subsection{The limit of the discriminant function}

Let $x_0\in\pi_1$ and consider the rescaled new observation
\begin{equation}
u_0=\frac{x_0-(\mu_1+\mu_2)/2}{\sqrt{\theta}}
=\eps_0\frac{\mu}{2\sqrt{\theta}}+\frac{v_0+\xi_0}{\sqrt{\theta}},
\quad\eps_0=+1,
\label{eq:6.1}
\end{equation}
where $v_0=\sum_{s\le m_1}\sqrt{\lambda_{1s}}z_{0s}h_{1s}$ and
$\xi_0=\sum_{s>m_1}\sqrt{\lambda_{1s}}z_{0s}h_{1s}$. From~\eqref{eq:5.4} we have
\begin{equation}
f(u_0)=\sum_{p\in\cI}\beta_p\gamma_p+b,
\quad\gamma_p=u_p^Tu_0.
\label{eq:6.2}
\end{equation}

\begin{theorem}
\label{thm:3}
Assume \textup{(A1)} to \textup{(A4)} and \textup{(S)}. Then, as $d\to\infty$,
it holds that
\begin{equation}
f(u_0)\wto f^*(z_0)
=\sum_{(i,j)\in\cI}\beta_{ij}^*\,g_{(ij),(0)}^*+b^*,
\label{eq:6.3}
\end{equation}
where $z_0=(z_{01},\dots,z_{0m_1})^T$, $(\beta^*,b^*)$ is given in
Theorem~\ref{thm:2}, and
\begin{equation}
\begin{aligned}
g_{(ij),(0)}^*
&=\frac{\eps_i\delta}{4}
+\frac{\eps_i\sqrt{\delta}}{2}\sum_{t\le m_1}\sqrt{c_{1t}}\,\beta_{1t}z_{0t}
+\frac{\sqrt{\delta}}{2}\sum_{s\le m_i}\sqrt{c_{is}}\,\beta_{is}z_{ijs}\\
&\qquad+\sum_{s\le m_i}\sum_{t\le m_1}\sqrt{c_{is}c_{1t}}\,r_{st}^{(i1)}z_{ijs}z_{0t}.
\end{aligned}
\label{eq:6.4}
\end{equation}
\end{theorem}

We emphasize that the term $c_{10}$ does not appear in~\eqref{eq:6.4}, that is, the
non-spiked part of $x_0$ does not contribute to the limit.

\begin{proof}
We follow the same steps as in the proof of Theorem~\ref{thm:1}, replacing
$u_{kl}$ by $u_0$. The only difference is that $x_0$ is independent of all the
training observations, and this single difference changes the conclusion.

\textit{Step~1.} From~\eqref{eq:4.1} and~\eqref{eq:6.1}, the bilinearity gives
\begin{equation}
\gamma_p=u_{ij}^Tu_0
=\underbrace{\frac{\eps_i\eps_0\Delta}{4\theta}}_{\mathrm{(I)}}
+\underbrace{\frac{\eps_i}{2}\frac{\mu^T(x_0-\mu_1)}{\theta}}_{\mathrm{(II)}}
+\underbrace{\frac{\eps_0}{2}\frac{\mu^T(x_{ij}-\mu_i)}{\theta}}_{\mathrm{(III)}}
+\underbrace{\frac{(x_{ij}-\mu_i)^T(x_0-\mu_1)}{\theta}}_{\mathrm{(IV)}},
\label{eq:6.5}
\end{equation}
where $p=(i,j)$ and we used $\eps_0=+1$ in~(I).

\textit{Step~2.} The term~(I) does not involve any random variable, and from~(A3)
we have $\eps_i\delta/4$, which is the first term of~\eqref{eq:6.4}.

\textit{Step~3.} Since $x_0\in\pi_1$, we may apply Lemma~\ref{lem:3} with $i=1$ to
$x_0$. Note that Lemma~\ref{lem:3} is a statement about a single observation and
does not use any relation to the other observations. Hence
\[
\frac{\mu^T(x_0-\mu_1)}{\theta}
\longrightarrow\sqrt{\delta}\sum_{t\le m_1}\sqrt{c_{1t}}\,\beta_{1t}z_{0t}.
\]
Multiplying by $\eps_i/2$, we obtain the second term of~\eqref{eq:6.4}.

\textit{Step~4.} Applying Lemma~\ref{lem:3} to $x_{ij}$ and using $\eps_0=+1$, we
obtain the third term of~\eqref{eq:6.4}.

\textit{Step~5.} We consider the term~(IV), which is the essential difference from
Theorem~\ref{thm:1}. Since $x_0$ is independent of all the training observations
$\{x_{ij}\}$, the term~(IV) is always the inner product of two distinct independent
observations, wherever $(i,j)$ is taken in $\cI$. The case of the inner product of
an observation with itself, which appeared in the diagonal components in
Theorem~\ref{thm:1}, never occurs here. We divide into two cases.

(IV-a) The case $i=1$. Although $x_{1j}$ and $x_0$ are observations from the same
population $\pi_1$, they are independent, so that we may apply~\eqref{eq:3.1} of
Lemma~\ref{lem:1}. We note that the condition $j\neq k$ in Lemma~\ref{lem:1} was
used only to guarantee the independence. Hence
\[
\frac{(x_{1j}-\mu_1)^T(x_0-\mu_1)}{\theta}
=\sum_{s=1}^{m_1}\frac{\lambda_{1s}}{\theta}z_{1js}z_{0s}+o_P(1).
\]
This agrees with the fourth term of~\eqref{eq:6.4} with $i=1$, because
$r_{st}^{(11)}=\delta_{st}$.

(IV-b) The case $i=2$. Since $x_{2j}$ and $x_0$ are independent observations from
different populations, we may apply Lemma~\ref{lem:2}. Exchanging the roles of the
two factors, we obtain
\[
\frac{(x_{2j}-\mu_2)^T(x_0-\mu_1)}{\theta}
=\sum_{s=1}^{m_2}\sum_{t=1}^{m_1}\frac{\sqrt{\lambda_{2s}\lambda_{1t}}}{\theta}
\rho_{ts}\,z_{2js}z_{0t}+o_P(1),
\]
where $\rho_{ts}=h_{2s}^Th_{1t}$. This agrees with the fourth term of
\eqref{eq:6.4} with $i=2$ under the convention $r_{st}^{(21)}=r_{ts}$.

\textit{Step~6.} We explain why $c_{10}$ does not appear. In Theorem~\ref{thm:1},
the term $c_{i0}$ came from~\eqref{eq:3.2}, which was used when $(i,j)=(k,l)$. The
corresponding term here is $\xi_{ij}^T\xi_0/\theta$. Since $\xi_{ij}$ and $\xi_0$
are independent, we have $E(\xi_{ij}^T\xi_0)=0$ and, by the same computation as in
\eqref{eq:3.3},
\[
E[(\xi_{ij}^T\xi_0)^2]=\tr(\Sigma_{i*}\Sigma_{1*})=o(\theta^2),
\]
so that $\xi_{ij}^T\xi_0/\theta=o_P(1)$ from Chebyshev's inequality.

Geometrically, the noise direction $\xi_0$ of the new observation is asymptotically
orthogonal to the noise directions $\xi_{ij}$ of the training data, because two
independent random directions are nearly orthogonal in high dimensions. On the
other hand, $\xi_{ij}$ is of course not orthogonal to itself, so that $c_{i0}$
remains only in the diagonal components in Theorem~\ref{thm:1}. Since the
discriminant function uses only the cross terms $u_p^Tu_0$, the noise of $x_0$
does not affect the classification.

\textit{Step~7.} From Steps~2 to~5, for each $p\in\cI$, the quantity $\gamma_p$ is
written as a continuous function of $Z^{(d)}$ plus $o_P(1)$. Since the number of
the components is $N$, which does not depend on $d$, we obtain
\begin{equation}
\gamma=(\gamma_p)_{p\in\cI}\wto\gamma^*=(g_{(ij),(0)}^*)_{(i,j)\in\cI}
\label{eq:6.6}
\end{equation}
from~(A4) together with the continuous mapping theorem and Slutsky's theorem.

\textit{Step~8.} We show the joint convergence of $(G,\gamma)$. The $N^2$
components of $G$ given in Theorem~\ref{thm:1} and the $N$ components of $\gamma$
given in Step~7 are all written as continuous functions of the common random
vector $Z^{(d)}$ plus $o_P(1)$. That is, there exists a continuous map $\Psi$,
which depends on the coefficients $c_{is}$, $\beta_{is}$ and $r_{st}$ converging
deterministically, such that $(G,\gamma)=\Psi(Z^{(d)})+o_P(1)$ uniformly on
compact sets. From~(A4) we have $Z^{(d)}\wto Z$, so that
\begin{equation}
(G,\gamma)\wto(G^*,\gamma^*)=\Psi(Z)
\label{eq:6.7}
\end{equation}
jointly. We note that this step is necessary, because the marginal convergences
shown separately do not imply the joint convergence.

\textit{Step~9.} We consider the map
\[
\Xi(G,\gamma)=\sum_p\beta_p(G)\gamma_p+b(G).
\]
From Lemma~\ref{lem:5}, the map $G\mapsto(\beta(G),b(G))$ is continuous on
$\cG_+$, and the inner product is continuous, so that $\Xi$ is continuous on
$\cG_+\times\R^N$. From~\eqref{eq:6.2} we have $f(u_0)=\Xi(G,\gamma)$. From
Theorem~\ref{thm:1} we have $P(G^*\succ O)=1$, so that the set of the discontinuity
points of $\Xi$ has measure zero under the limiting distribution. Thus we obtain
\eqref{eq:6.3} from~\eqref{eq:6.7} and the continuous mapping theorem.
\end{proof}

\begin{table}[ht]
\centering
\caption{Correspondence between the proofs of Theorems~\ref{thm:1} and~\ref{thm:3}.}
\label{tab:1}
\begin{tabular}{@{}lll@{}}
\toprule
Term & Theorem~\ref{thm:1} (training vs training)
& Theorem~\ref{thm:3} (training vs new) \\
\midrule
(I) deterministic
& $\eps_i\eps_k\delta/4$
& $\eps_i\delta/4$ (since $\eps_0=1$) \\
(II), (III) linear
& Lemma~\ref{lem:3} on both sides
& Lemma~\ref{lem:3} on both sides, one being $z_0$ \\
(IV), distinct
& \eqref{eq:3.1} or Lemma~\ref{lem:2}
& always this case \\
(IV), identical
& \eqref{eq:3.2}, giving $+c_{i0}$
& never occurs, no $c_{i0}$ \\
\bottomrule
\end{tabular}
\end{table}

\subsection{The inconsistency}

We now derive the inconsistency of the SVM. Here we need to derive a lower bound of
the expectation from the convergence in distribution. For this purpose we use
Fatou's lemma, which we summarize in Appendix~\ref{app:A.2} together with its
intuitive meaning. Roughly speaking, the lemma states that a non-negative mass may
escape in the limiting procedure but never emerges, and this is exactly the
direction of the inequality which we need.

\begin{corollary}
\label{cor:1}
Assume \textup{(A1)} to \textup{(A4)} and \textup{(S)}. Assume that $z_0$ has a
continuous distribution whose support is the whole space. Let
$e(1)=P\{f(u_0)<0\mid\text{training data}\}$. Then it holds that
\[
e(1)\wto e^*(1)=P_{z_0}\{f^*(z_0)<0\}.
\]
Moreover, if $f^*$ is a non-constant function of $z_0$ with positive probability,
then
\[
\liminf_{d\to\infty}E\{e(1)+e(2)\}>0,
\]
that is, the hard-margin SVM and the BC-SVM do not hold the consistency property.
\end{corollary}

\begin{proof}
\textit{Step~1.} From Theorem~\ref{thm:3} we have $f(u_0)\wto f^*(z_0)$. Since
$z_0$ is independent of the training data and the distribution of $f^*(z_0)$ is
continuous, the boundary $\{f^*=0\}$ has measure zero under the limiting
distribution. Hence $e(1)\wto e^*(1)$ from the Portmanteau theorem.

\textit{Step~2.} From~\eqref{eq:6.3} and~\eqref{eq:6.4}, $f^*(z_0)$ is a polynomial
in $z_0$ consisting of an affine part and a quadratic part. If it is non-constant,
then it takes negative values on a set of positive Lebesgue measure, and hence
$E\{e^*(1)\}>0$ because the support of $z_0$ is the whole space.

\textit{Step~3.} Since $e(1)\ge 0$, we may apply Lemma~\ref{lem:7} in
Appendix~\ref{app:A.2} and obtain
\[
E\{e^*(1)\}\le\liminf_{d\to\infty}E\{e(1)\}.
\]
Adding $e(2)$ does not change the inequality.
\end{proof}

\begin{remark}
\label{rem:5}
We emphasize that the inconsistency in Corollary~\ref{cor:1} is different from the
strong inconsistency of \citet{nakayama2017}. They showed that the
misclassification rate tends to one when the sample sizes are imbalanced, which is
caused by the deterministic bias $\kappa$ in~\eqref{eq:1.3}. Here the
misclassification rate tends neither to zero nor to one, and it is caused by the
randomness which survives in the limit. Hence the bias correction alone cannot
remove this difficulty. We discuss this point in Sections~\ref{sec:bias}
and~\ref{sec:scsvm}.
\end{remark}

\section{An example with an explicit limit}
\label{sec:example}

In this section, we give an example in which the limiting misclassification rate is
obtained in a closed form. We consider the following setting.

\begin{enumerate}[label=(E),leftmargin=*]
\item Let $n_1=n_2=1$, $\theta_1=\theta_2=\theta$ and $m_1=m_2=1$. The two
populations have a common spiked direction $h$ with
$\lambda_{11}=\lambda_{21}=\lambda=c\theta$, so that $r_{11}=1$. The spike is
orthogonal to the mean difference, that is, $h\perp\mu$ or equivalently
$\beta_{11}=\beta_{21}=0$. We have $c_{10}=c_{20}=1-c$ and $z$'s are standard
normal.
\end{enumerate}

\begin{theorem}
\label{thm:4}
Assume \textup{(E)}. Then, for $x_0\in\pi_1$, it holds that
\begin{equation}
\frac{y(x_0)}{\theta}
\wto
\frac{\delta}{2}+c(z_1-z_2)\Bigl(z_0-\frac{z_1+z_2}{2}\Bigr),
\label{eq:7.1}
\end{equation}
and, putting $A=c(z_1-z_2)$, the limiting conditional misclassification rate is
given by
\begin{equation}
e^*(1)=\Phi\Bigl(\sgn(A)\frac{z_1+z_2}{2}-\frac{\delta}{2|A|}\Bigr),
\label{eq:7.2}
\end{equation}
where $\Phi$ denotes the standard normal distribution function and $z_1,z_2$ are
i.i.d.\ standard normal.
\end{theorem}

\begin{proof}
(i) When each class has one observation, the hard-margin SVM is the perpendicular
bisector of the two points. Indeed, putting $w=x_{11}-x_{21}$ and
$b=-w^T(x_{11}+x_{21})/2$, we have $y_p(w^Tx_p+b)=1$ for $p\in\cI$, and the
maximality of the margin is easily checked. Hence
\begin{equation}
y(x_0)=(x_{11}-x_{21})^Tx_0-\frac{\|x_{11}\|^2-\|x_{21}\|^2}{2},
\label{eq:7.3}
\end{equation}
where we used $w^Tx_0+b=(x_{11}-x_{21})^Tx_0-\|x_{11}\|^2/2+\|x_{21}\|^2/2$.

(ii) We take the origin at $(\mu_1+\mu_2)/2$, so that $\mu_1=\mu/2$ and
$\mu_2=-\mu/2$. Under~(E) we have
\[
x_{11}=\tfrac{\mu}{2}+\sqrt{\lambda}\,z_1 h+\xi_1,\quad
x_{21}=-\tfrac{\mu}{2}+\sqrt{\lambda}\,z_2 h+\xi_2,\quad
x_0=\tfrac{\mu}{2}+\sqrt{\lambda}\,z_0 h+\xi_0,
\]
so that $(x_{11}-x_{21})^Tx_0=\Delta/2+\lambda(z_1-z_2)z_0+o_P(\theta)$. Using
$h\perp\mu$ and $h\perp\xi_i$, and noting that all the cross terms are $o_P(\theta)$
from Lemmas~\ref{lem:2} and~\ref{lem:3} and Step~6 of the proof of
Theorem~\ref{thm:3}, we obtain
\[
(x_{11}-x_{21})^Tx_0=\frac{\Delta}{2}+\lambda(z_1-z_2)z_0+o_P(\theta).
\]
On the other hand, from~\eqref{eq:3.2} we have
\[
\|x_{11}\|^2=\frac{\Delta}{4}+\lambda z_1^2+\tr(\Sigma_{1*})+o_P(\theta),\quad
\|x_{21}\|^2=\frac{\Delta}{4}+\lambda z_2^2+\tr(\Sigma_{2*})+o_P(\theta).
\]
Here $\tr(\Sigma_{1*})$ and $\tr(\Sigma_{2*})$ cancel out because $n_1=n_2$ and
$\theta_1=\theta_2$. We note that, in general, the bias term of the existing theory
appears at this place. Hence
\[
\frac{\|x_{11}\|^2-\|x_{21}\|^2}{2}=\frac{\lambda(z_1^2-z_2^2)}{2}+o_P(\theta).
\]
Substituting these into~\eqref{eq:7.3} and dividing by $\theta$, we obtain
\[
\frac{y(x_0)}{\theta}
=\frac{\delta}{2}+c(z_1-z_2)z_0-\frac{c(z_1^2-z_2^2)}{2}+o_P(1),
\]
and the factorization $z_1^2-z_2^2=(z_1-z_2)(z_1+z_2)$ gives~\eqref{eq:7.1}.

(iii) Conditioning on $z_1$ and $z_2$, we have $z_0\sim N(0,1)$. Let
$A=c(z_1-z_2)$. When $A>0$, we have $f^*<0$ if and only if
$z_0<\sgn(A)(z_1+z_2)/2-\delta/(2|A|)$, so that
$e^*(1)=\Phi(\sgn(A)(z_1+z_2)/2-\delta/(2|A|))$. When $A<0$, the inequality is
reversed and
\[
e^*(1)=\Phi\Bigl(\sgn(A)\frac{z_1+z_2}{2}-\frac{\delta}{2|A|}\Bigr)
\]
still holds. Combining the two cases with $\Phi(-\cdot)=1-\Phi(\cdot)$, we obtain
\eqref{eq:7.2}.
\end{proof}

\subsection{Interpretation}

The right-hand side of~\eqref{eq:7.1} is the sum of the signal $\delta/2$ coming
from the mean difference and the nearest centroid rule
$c(z_1-z_2)(z_0-(z_1+z_2)/2)$ in the one-dimensional spiked coordinate. That is,
the high dimensionality vanishes except for the spiked direction, and what remains
is a classical small-sample problem in one dimension with $N=2$ observations. We
give the following observations.

\begin{enumerate}[label=(\roman*)]
\item When $c\to 0$, which corresponds to the non-spiked case, we have $A\to 0$ and
hence $e^*(1)\to 0$ when $\delta>0$, so that $E\{e^*(1)\}\to 0$. That is, the
consistency property of the existing theory is recovered.

\item When $\delta=0$, we have $e^*(1)=\Phi(\sgn(A)(z_1+z_2)/2)$, whose expectation
is $1/2$. That is, the classification is completely random.

\item The misclassification rate increases as $c$ increases. We emphasize that the
spike degrades the performance even though it is orthogonal to the mean difference,
so that it carries no information about the classification. This phenomenon cannot
be observed in the existing theory.
\end{enumerate}

\section{Modification of the bias term}
\label{sec:bias}

The BC-SVM of \citet{nakayama2017} subtracts the bias term $\kappa$ in
\eqref{eq:1.3}, which comes from the geometric representation. On the other hand,
Theorem~\ref{thm:1} shows that the deterministic part of the limiting Gram matrix
is the diagonal matrix $D$ whose components are $c_{i0}=\lim\tr(\Sigma_{i*})/\theta$,
that is, only the trace of the non-spiked part. The spiked part
$\sum_{s\le m_i}\lambda_{is}$ does not have a deterministic limit.

\begin{proposition}
\label{prop:1}
Assume \textup{(A1)} to \textup{(A4)} and \textup{(S)}. Then the deterministic
offset term of the discriminant function is given by
\[
\kappa_*=\frac{\tr(\Sigma_{1*})}{n_1}-\frac{\tr(\Sigma_{2*})}{n_2}
=\kappa-\Biggl(
\frac{1}{n_1}\sum_{s\le m_1}\lambda_{1s}
-\frac{1}{n_2}\sum_{s\le m_2}\lambda_{2s}
\Biggr).
\]
Hence the BC-SVM, which uses an estimator of $\tr(\Sigma_i)$ including the spikes,
overcorrects the bias by
$\sum_{s\le m_1}\lambda_{1s}/n_1-\sum_{s\le m_2}\lambda_{2s}/n_2$, which is of
order $\theta$ and is not negligible compared with the signal $\Delta=O(\theta)$.
\end{proposition}

\begin{proof}[Outline of proof]
We expand the discriminant function for general $n_1$ and $n_2$ in the same way as
\eqref{eq:7.3}. Since $w=\sum_p\alpha_p y_p u_p$, the deterministic term of the
discriminant function comes from the diagonal part $D$ of Theorem~\ref{thm:1}
through the Karush--Kuhn--Tucker conditions, and it has the form
$\sum_p\alpha_p^* y_p\cdot(\text{diagonal of }D)$. In the existing theory,
$\theta_i$'s are equalized within each class and this gives the form $\kappa$ in
\eqref{eq:1.3}. Under the spiked model, the diagonal term is
$\tr(\Sigma_{i*})/\theta$ from~\eqref{eq:3.2} and not $\tr(\Sigma_i)/\theta$, so
that the replacement occurs. The difference is
$\tr(\Sigma_i)-\tr(\Sigma_{i*})=\sum_{s\le m_i}\lambda_{is}$ by definition.
\end{proof}

\begin{remark}
\label{rem:6}
Spiked eigenvalue structures are commonly observed in actual high-dimensional data
such as gene expression data. Proposition~\ref{prop:1} suggests that, when the
BC-SVM is applied to such data, the bias correction works too strongly and the
separating hyperplane is shifted excessively towards the minority class. We can
expect an improvement by replacing the estimator of $\tr(\Sigma_i)$ with
$\tr(\Sigma_i)-\sum_{s\le\hat m_i}\tilde\lambda_{is}$, where $\tilde\lambda_{is}$
denotes the estimator of the spiked eigenvalue given by the NR methodology.
\end{remark}

\section{The spike-corrected SVM}
\label{sec:scsvm}

Corollary~\ref{cor:1} shows that the SVM does not hold the consistency property
under the spiked model. The interpretation in Section~\ref{sec:example} gives a
remedy: a spiked direction is a source of noise which cannot be learned when the
sample size is fixed, so that we should remove it if it carries no information about
the classification.

\begin{definition}
\label{def:1}
For a direction $h_{is}$, we define the signal-to-spike ratio by
\[
\eta_{is}=\frac{(h_{is}^T\mu)^2}{\lambda_{is}}=\frac{\Delta b_{is}^2}{\lambda_{is}}
\longrightarrow\frac{\delta\beta_{is}^2}{c_{is}}.
\]
\end{definition}

Under~(A3) and~(S), it holds that $\eta_{is}\to\delta\beta_{is}^2/c_{is}$.

\begin{proposition}
\label{prop:2}
Consider the subproblem obtained by projecting the data onto a one-dimensional
direction $h$. When $n_i$ is fixed, the limiting misclassification rate obtained by
using this direction is not smaller than $\Phi(-\sqrt{\eta}/2)$. Hence, when
$\eta_{is}\to\infty$, the direction should be kept, and when $\eta_{is}=O(1)$, the
direction should be removed.
\end{proposition}

\begin{proof}[Outline of proof]
Projecting onto $h$, the population $\pi_i$ becomes a one-dimensional distribution
with mean $h^T\mu_i$ and variance $\lambda$. The mean difference is $|h^T\mu|$ and
the standard deviation is $\sqrt{\lambda}$, so that the error rate of the Bayes
rule with known parameters is $\Phi\{-|h^T\mu|/(2\sqrt{\lambda})\}=\Phi(-\sqrt{\eta}/2)$
in the normal case. When $\eta=O(1)$, this is bounded below by a positive constant,
and the estimation error makes it larger. On the other hand, when $\eta\to\infty$,
it tends to zero. As for the loss caused by the removal, we note that
$\eta_{is}=O(1)$ implies $(h_{is}^T\mu)^2=O(\lambda_{is})=O(\theta)$, so that the
loss is negligible when the residual signal diverges faster than $\theta$.
\end{proof}

\subsection{The procedure}

\begin{definition}[SC-SVM]
\label{def:2}
We propose the spike-corrected SVM as follows.
\begin{enumerate}
\item Split the observations from each $\pi_i$ into $D_i^{(1)}$ for estimation and
$D_i^{(2)}$ for training. The split is essential in the theory because it makes
the estimated projection independent of the training data.
\item Apply the NR methodology of \citet{yata2012} to $D_i^{(1)}$ and obtain the
number of the spikes $\hat m_i$, the eigenvalues $\tilde\lambda_{is}$ and the
eigenvectors $\hat h_{is}$.
\item Compute $\hat\eta_{is}=(\hat h_{is}^T\hat\mu)^2/\tilde\lambda_{is}$ with
$\hat\mu=\bar x_1^{(1)}-\bar x_2^{(1)}$, and collect the directions with
$\hat\eta_{is}\le\tau_d$ into the set $\hat S$, where $\tau_d\to\infty$ is a
threshold such as $\tau_d=\log d$.
\item Let $\hat U$ be an orthonormal basis of $\hat S$ and put
$\hat P=I_d-\hat U\hat U^T$. Project the training data and the new observation
onto $\hat P$.
\item Apply the BC-SVM with the modified bias term $\hat\kappa_*$ given in
Proposition~\ref{prop:1} to the projected data. When some spiked directions are
kept, construct a plug-in rule on the low-dimensional part and add it to the
discriminant function.
\end{enumerate}
\end{definition}

We assume the following conditions.

\begin{enumerate}[label=(C\arabic*),leftmargin=*]
\item It holds that $n_i\to\infty$ and $n_i/d\to 0$.
\item For $s\le m_i$, it holds that $\tilde\lambda_{is}/\lambda_{is}\to 1$ and
$\hat h_{is}^Th_{is}\to 1$ in probability.
\item Let $P_*$ be the projection onto the orthogonal complement of the removed
spikes and $\Delta_*=\|P_*\mu\|^2$. It holds that
\[
\Delta_*/\{\tr(\Sigma_{1*}^2)+\tr(\Sigma_{2*}^2)\}^{1/2}\to\infty.
\]
\end{enumerate}

The condition~(C2) holds under the conditions given by \citet{yata2012}, such as
$n_i\tr(\Sigma_{i*}^2)/\lambda_{im_i}^2\to 0$. We assume it explicitly in this
paper.

\begin{theorem}
\label{thm:5}
Assume \textup{(A1)} to \textup{(A4)}, \textup{(S)} and \textup{(C1)} to
\textup{(C3)}. Then the SC-SVM holds the consistency property, that is,
\[
e_{SC}(1)+e_{SC}(2)\longrightarrow 0
\quad\text{in probability as }d\to\infty.
\]
On the other hand, from Corollary~\ref{cor:1}, the SVM and the BC-SVM do not hold
the consistency property under the same conditions.
\end{theorem}

\begin{proof}
(i) We first consider the ideal procedure in which $\hat P$ is replaced by the true
projection $P_*$. The eigenvalues of $\Sigma_i^P=P_*\Sigma_i P_*$ consist of
$\{\lambda_{is}:s\notin\text{removed}\}$ and the kept spikes. For the part from
which the noise spikes are removed, we have
\[
\frac{\tr\{(\Sigma_i^P)^2\}}{\{\tr(\Sigma_i^P)\}^2}\to 0
\]
from~(S). That is, the projected data meet the condition~\eqref{eq:1.2} of the
non-spiked model.

(ii) Under the non-spiked model, the geometric representation is recovered, so that
the consistency of the BC-SVM given by \citet{nakayama2017} is applicable under
(C3). We use the modified bias term $\kappa_*$ given in Proposition~\ref{prop:1}.

(iii) We replace $P_*$ by $\hat P$. From~(C2) we have $\|\hat P-P_*\|=o_P(1)$.
Since $\hat P$ is independent of the training data and $x_0$ from the sample
splitting, we may argue conditionally. Noting that the spiked coordinates of $x$
are of order $O_P(\sqrt{\theta})$, we have
\[
\|(\hat P-P_*)x\|=o_P(\sqrt{\theta}),
\]
so that the contribution to the inner products is $o_P(\theta)$. Hence it is
absorbed into the remainder terms of Lemmas~\ref{lem:1} to~\ref{lem:3} and the
conclusions of~(i) and~(ii) are preserved.

(iv) From~(C2) and $n_i\to\infty$, the quantity $\hat\eta_{is}$ is relatively
consistent for $\eta_{is}$. Choosing $\tau_d\to\infty$ such that
$\tau_d=o(\eta_{is})$ for the directions to be kept, the selection is correct with
probability tending to one.
\end{proof}

\section{The case that the sample size is fixed}
\label{sec:fixed}

Theorem~\ref{thm:5} requires $n_i\to\infty$. We show that this is not a technical
restriction but is essential.

\begin{proposition}
\label{prop:3}
Assume \textup{(E)} with $c\in(0,1)$ fixed and let $N$ be fixed. Consider any
procedure which uses a projection $\hat P$ of fixed rank estimated from the
training data. Then the angle between $\hat h$ and the true spiked direction $h$
does not tend to zero, that is, $|\hat h^Th|\to\varrho<1$ in probability. Hence
the residual spiked component of the new observation after the projection
satisfies
\[
\bigl\|(I-\hat h\hat h^T)\sqrt{\lambda}\,z_0 h\bigr\|^2
=\lambda z_0^2(1-\varrho^2)+o_P(\theta)=O_P(\theta),
\]
so that the spike cannot be removed by the projection and the projection-based
SC-SVM does not hold the consistency property.
\end{proposition}

\begin{proof}[Outline of proof]
The region $\lambda/\tr(\Sigma_*)\to c/(1-c)\in(0,\infty)$ corresponds to the
subspace inconsistency region in the HDLSS PCA theory of \citet{jung2009}, in
which the sample principal component direction is not consistent and the angle has
a non-degenerate limit. Intuitively, the information about the spiked direction
contained in the sample consists of the $N$ coefficients $\{z_{ijs}\}$ only, on
which the isotropic noise of order $\sqrt{\tr(\Sigma_*)}$ is superposed, so that
the signal-to-noise ratio remains a constant when $N$ is fixed. Since the
estimation error of $\hat h$ remains at a constant proportion, the spiked
component $\sqrt{\lambda}\,z_0 h$ of $x_0$, whose magnitude is of order
$\sqrt{\theta}$, also remains at a constant proportion after the projection. Then
the argument of Theorem~\ref{thm:3} is applicable and the misclassification rate
does not tend to zero.
\end{proof}

\begin{remark}
\label{rem:7}
Proposition~\ref{prop:3} is a statement about the class of projection-based
procedures. We have not proved a minimax type lower bound which states that any
classification rule based on the training data does not hold the consistency
property when $N$ is fixed and $\delta<\infty$. We expect that it is reduced to the
positivity of the Bayes risk in the limiting experiment given in
Theorem~\ref{thm:1}, but it remains open. On the other hand, when $\delta\to\infty$,
we can see from~\eqref{eq:7.2} that the consistency property is recovered even for
fixed $N$.
\end{remark}

\section{Numerical simulations}
\label{sec:sim}

In this section, we check the performance of the classifiers and the validity of
the theoretical results by numerical simulations. We give five experiments. In
Section~\ref{sec:sim-thm4}, we verify Theorem~\ref{thm:4} in the setting~(E). In
Section~\ref{sec:sim-gram}, we verify the convergence of the Gram matrix given in
Theorem~\ref{thm:1}. In Section~\ref{sec:sim-inconsist}, we check the inconsistency
given in Corollary~\ref{cor:1}. In Section~\ref{sec:sim-bias}, we check the bias
term given in Proposition~\ref{prop:1}. In Section~\ref{sec:sim-classifiers}, we
compare the classifiers including the SC-SVM.

Throughout this section, the populations are Gaussian and we take the coordinates
so that the first coordinate is the direction of $\mu$, the next $m_i$ coordinates
are the spiked directions, which are orthogonal to $\mu$, and the remaining
coordinates give the isotropic non-spiked part. That is,
\begin{equation}
\Sigma_i=\diag(\sigma_i^2,\lambda_{i1},\dots,\lambda_{im_i},\sigma_i^2,\dots,\sigma_i^2),
\quad
\lambda_{is}=c_{is}\theta,
\quad
\theta=d,
\label{eq:11.1}
\end{equation}
where $\sigma_i^2$ is determined by $\tr(\Sigma_{i*})=\theta\,c_{i0}$. Since
$\Sigma_i$ is diagonal, we can compute the conditional misclassification rates
exactly for each realization of the training data, so that we do not need a test
sample. We give the standard errors of the averages over the replications.
The code for each experiment is distributed with this manuscript as a separate
Python script (\texttt{spiked\_svm\_sim.py}, \texttt{spiked\_svm\_gram.py},
\texttt{spiked\_svm\_inconsistency.py}, \texttt{spiked\_svm\_bias.py},
\texttt{spiked\_svm\_classifiers.py}), together with the shared utilities in
\texttt{spiked\_svm\_common.py}.

\subsection{Verification of Theorem~\ref{thm:4}}
\label{sec:sim-thm4}

We first consider the setting~(E) with $d=20000$, in which the limiting
misclassification rate is given by~\eqref{eq:7.2} in a closed form. We considered
$c\in\{0.1,0.3,0.5,0.8\}$ and $\delta\in\{0,1,4\}$, and took the average over
$600$ replications. We computed the limiting value~\eqref{eq:7.2} by the Monte
Carlo integration with $2\times 10^5$ points. The implementation is given in
\texttt{spiked\_svm\_sim.py}. We give the results in
Table~\ref{tab:2}.

\begin{table}[ht]
\centering
\caption{The limiting misclassification rate~\eqref{eq:7.2} and the simulated
$E\{e(1)\}$ for $d=20000$ in the setting~(E).}
\label{tab:2}
\begin{tabular}{@{}rrrrr@{}}
\toprule
$c$ & $\delta$ & limit~\eqref{eq:7.2} & simulation & difference \\
\midrule
0.1 & 0.0 & 0.5009 & 0.4993 & 0.0016 \\
0.1 & 1.0 & 0.0112 & 0.0122 & 0.0010 \\
0.1 & 4.0 & 0.0000 & 0.0000 & 0.0000 \\
0.3 & 0.0 & 0.5009 & 0.4998 & 0.0011 \\
0.3 & 1.0 & 0.1093 & 0.1091 & 0.0002 \\
0.3 & 4.0 & 0.0039 & 0.0047 & 0.0008 \\
0.5 & 0.0 & 0.5009 & 0.4997 & 0.0012 \\
0.5 & 1.0 & 0.1833 & 0.1832 & 0.0001 \\
0.5 & 4.0 & 0.0215 & 0.0225 & 0.0010 \\
0.8 & 0.0 & 0.5009 & 0.4997 & 0.0012 \\
0.8 & 1.0 & 0.2517 & 0.2519 & 0.0002 \\
0.8 & 4.0 & 0.0598 & 0.0602 & 0.0004 \\
\bottomrule
\end{tabular}
\end{table}

We observe that the differences are smaller than $0.002$ for all the cases, which
are within the Monte Carlo error whose standard error is about $0.006$ for $600$
replications. Hence Theorem~\ref{thm:4} is confirmed numerically. When
$\delta=0$, the misclassification rate is about $0.5$ for all $c$, so that the
classification is completely random. When $\delta$ is fixed, the misclassification
rate increases monotonically as $c$ increases. We emphasize that the spike is
orthogonal to $\mu$ in this setting, so that it carries no information about the
classification. Nevertheless it degrades the performance of the SVM severely.

\subsection{Convergence of the Gram matrix}
\label{sec:sim-gram}

Next, we check Theorem~\ref{thm:1} directly. We considered $n_1=n_2=2$, $m_i=1$,
$c=0.5$ and $\delta=1$, and took $d$ from $10^2$ to $10^6$. In each replication,
we first generated the spiked coordinates $\{z_{ij1}\}$ and fixed them, and then
generated the non-spiked part for the given $d$. We computed the limiting Gram
matrix $G^*$ from~\eqref{eq:4.3} using the same $\{z_{ij1}\}$. Since the spike is
orthogonal to $\mu$, we have $\beta_{is}=0$ and $r_{11}=1$, so that
\[
g^*_{(ij),(kl)}
=\frac{\eps_i\eps_k\delta}{4}+c\,z_{ij1}z_{kl1},
\qquad
G^*=G_0^*+(1-c)I_N.
\]
We measured $\|G-G^*\|_F$ and took the average over $100$ replications. The
implementation is given in \texttt{spiked\_svm\_gram.py}. We give
the results in Table~\ref{tab:gram}.

\begin{table}[ht]
\centering
\caption{The Frobenius norm $\|G-G^*\|_F$ and the ratio of the successive values.}
\label{tab:gram}
\begin{tabular}{@{}rrrr@{}}
\toprule
$d$ & $\|G-G^*\|_F$ & s.e. & ratio to previous \\
\midrule
100 & 0.29956 & 0.00783 & -- \\
1{,}000 & 0.09099 & 0.00245 & 3.292 \\
10{,}000 & 0.03138 & 0.00077 & 2.900 \\
100{,}000 & 0.01039 & 0.00027 & 3.021 \\
1{,}000{,}000 & 0.00299 & 0.00007 & 3.471 \\
\bottomrule
\end{tabular}
\end{table}

We observe that $\|G-G^*\|_F$ tends to zero. Since the dimension is multiplied by
$10$ in each row, the ratio should be $\sqrt{10}\approx 3.162$ if the error is of
order $d^{-1/2}$. The observed ratios are $3.292$, $2.900$, $3.021$ and $3.471$,
which are close to $\sqrt{10}$. Hence Theorem~\ref{thm:1} is confirmed, and
moreover the convergence rate is of order $d^{-1/2}$ in this setting. We note that
the rate itself is not claimed in Theorem~\ref{thm:1}, which gives the convergence
only.

\subsection{The inconsistency}
\label{sec:sim-inconsist}

We check Corollary~\ref{cor:1}. We considered $n_1=n_2=2$, $m_i=1$ and
$\delta=0.25$, and took $c\in\{0.8,0.5,0.2\}$ together with the non-spiked case in
which $c=10^{-4}$. We took the average of $e(1)+e(2)$ over $200$ replications for
$d\in\{250,1000,4000,16000,64000\}$. The implementation is given in
\texttt{spiked\_svm\_inconsistency.py}. We give the results in
Figure~\ref{fig:inconsist}.

\begin{figure}[ht]
\centering
\includegraphics[width=0.72\textwidth]{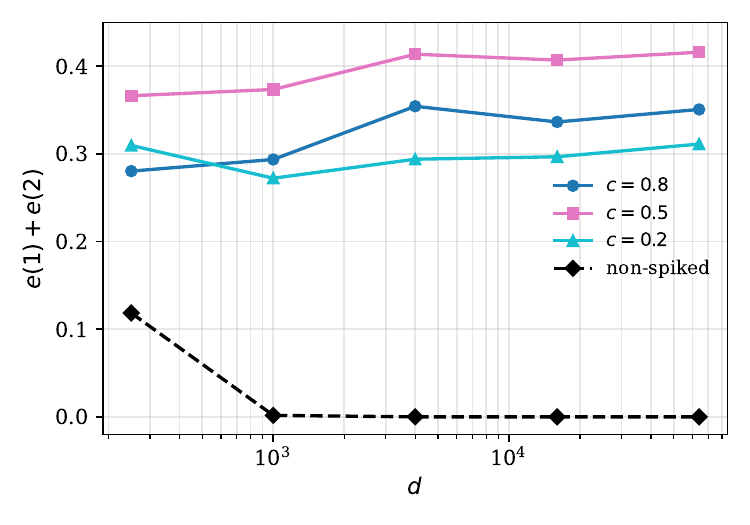}
\caption{Average of $e(1)+e(2)$ for the hard-margin SVM vs.\ $d$
($n_1=n_2=2$, $\delta=0.25$).}
\label{fig:inconsist}
\end{figure}

In the non-spiked case, the misclassification rate tends to zero rapidly as $d$
increases, which is the consistency property of the existing theory. On the other
hand, under the spiked model, the misclassification rate does not decrease at all
even though $d$ is multiplied by $256$. Hence Corollary~\ref{cor:1} is confirmed.

We also observe that the misclassification rate is not monotone in $c$, since the
value for $c=0.5$ is larger than the one for $c=0.8$. We note that the
monotonicity was observed in Table~\ref{tab:2}, in which $n_1=n_2=1$. Hence the
monotonicity does not hold in general when $n_i\ge 2$.

We give a supplementary result in order to clarify the practical meaning of
Corollary~\ref{cor:1}. We took $d=4000$ and $\delta=0.25$, and varied
$n_1=n_2=n$. We give the results in Table~\ref{tab:n-vary}.

\begin{table}[ht]
\centering
\caption{The average of $e(1)+e(2)$ for $d=4{,}000$ and $\delta=0.25$.}
\label{tab:n-vary}
\begin{tabular}{@{}rrrr@{}}
\toprule
$n$ & $c=0.8$ & $c=0.5$ & non-spiked \\
\midrule
2 & 0.2684 & 0.3467 & 0.0000 \\
3 & 0.0771 & 0.1589 & 0.0000 \\
5 & 0.0154 & 0.0424 & 0.0000 \\
10 & 0.0000 & 0.0003 & 0.0000 \\
\bottomrule
\end{tabular}
\end{table}

Although Corollary~\ref{cor:1} states that the limit is positive for any fixed $n$,
the value decreases rapidly as $n$ increases and it is already smaller than
$10^{-3}$ for $n=10$ in this setting. Hence the practical effect of the
inconsistency is restricted to the case that the sample size is very small. We
note that the setting of Table~\ref{tab:n-vary} is favourable in the sense that
$\delta$ is not small. We give a harder setting in
Section~\ref{sec:sim-classifiers}.

\subsection{The bias term}
\label{sec:sim-bias}

We check Proposition~\ref{prop:1}. In the $u$-scale, the midpoint of $\mu_1$ and
$\mu_2$ is the origin, so that the intercept $b$ is exactly the value of the
discriminant function at the midpoint, that is, the deterministic offset. We
measured $b$ and compared it with the predictions given by $\kappa$ and by
$\kappa_*$.

From the outline of the proof of Proposition~\ref{prop:1}, the offset is of the
form $-A\kappa_*/(2\theta)$. Writing $A=\sum_j\alpha_{1j}=\sum_j\alpha_{2j}$,
which holds from the constraint $y^T\alpha=0$, and noting that the dual weights are
equalized within each class, the prediction is
\begin{equation}
b\approx -\frac{A}{2}\cdot\frac{\kappa_*}{\theta}
\quad\text{instead of}\quad
b\approx -\frac{A}{2}\cdot\frac{\kappa}{\theta}.
\label{eq:11.2}
\end{equation}
We note that $A$ is observable, so that~\eqref{eq:11.2} gives a quantitative
prediction without any unknown constant. In order to distinguish the two
predictions, we chose the configurations in which $\kappa$ and $\kappa_*$ disagree
qualitatively. Since $\theta_1=\theta_2=\theta$ in our setting, we have
\[
\frac{\kappa}{\theta}=\frac{1}{n_1}-\frac{1}{n_2},
\qquad
\frac{\kappa_*}{\theta}=\frac{1-c_1}{n_1}-\frac{1-c_2}{n_2},
\]
so that we can make $\kappa=0$ while $\kappa_*\neq 0$ by taking $n_1=n_2$ and
$c_1\neq c_2$, and we can make $\kappa_*=0$ while $\kappa\neq 0$ by a suitable
choice. We took $d=8000$, $\delta=1$ and $300$ replications. The implementation
is given in \texttt{spiked\_svm\_bias.py}. We give the results
in Table~\ref{tab:bias}.

\begin{table}[ht]
\centering
\caption{The measured intercept $b$ and the two predictions in~\eqref{eq:11.2}.}
\label{tab:bias}
\setlength{\tabcolsep}{3pt}
\begin{tabular}{@{}rrrrrrrrrr@{}}
\toprule
$n_1$ & $n_2$ & $c_1$ & $c_2$ & $\kappa/\theta$ & $\kappa_*/\theta$
  & $b$ (s.e.) & $A$ & pred.\ $\kappa$ & pred.\ $\kappa_*$ \\
\midrule
5 & 5 & 0.2 & 0.8 & 0.0000 & 0.1200 & $-0.0899$ (0.0017) & 1.644 & 0.0000 & $-0.0987$ \\
4 & 4 & 0.1 & 0.7 & 0.0000 & 0.1500 & $-0.0921$ (0.0029) & 1.492 & 0.0000 & $-0.1119$ \\
10 & 5 & 0.4 & 0.7 & $-0.1000$ & 0.0000 & $0.0071$ (0.0011) & 1.766 & 0.0883 & 0.0000 \\
8 & 4 & 0.3 & 0.65 & $-0.1250$ & 0.0000 & $0.0103$ (0.0015) & 1.669 & 0.1043 & 0.0000 \\
8 & 4 & 0.5 & 0.5 & $-0.1250$ & $-0.0625$ & $0.0557$ (0.0012) & 1.657 & 0.1035 & 0.0518 \\
\bottomrule
\end{tabular}
\end{table}

We observe the following. In the first two rows, we have $\kappa=0$, so that the
prediction by $\kappa$ is zero. The measured values are $-0.0899$ and $-0.0921$,
which are far from zero compared with the standard errors, while the predictions
by $\kappa_*$ are $-0.0987$ and $-0.1119$. In the third and the fourth rows, we
have $\kappa_*=0$, and the measured values are $0.0071$ and $0.0103$, which are
much smaller than the predictions by $\kappa$, that is, $0.0883$ and $0.1043$. In
the last row, in which both are non-zero, the measured value $0.0557$ is close to
the prediction by $\kappa_*$, that is, $0.0518$, while the prediction by $\kappa$
is $0.1035$, which is about twice as large.

Hence Proposition~\ref{prop:1} is confirmed. The remaining discrepancy, which is
at most about $0.02$ in the first two rows and which is not zero in the third and
the fourth rows, is considered to come from the finite-dimensional effect and from
the fact that the dual weights are not exactly equalized within each class. We
note that the qualitative conclusion is not affected by this discrepancy, because
the two predictions differ in sign or in magnitude by a factor of two.

\subsection{Comparison of the classifiers}
\label{sec:sim-classifiers}

Finally, we compare the classifiers. We consider the following procedures. For
the bias correction, we use the form suggested by the outline of the proof of
Proposition~\ref{prop:1}, that is, we shift the intercept by
\begin{equation}
b\;\longmapsto\;
b+\frac12\bigl(\bar\alpha_1\,\hat c_{10}-\bar\alpha_2\,\hat c_{20}\bigr),
\label{eq:11.3}
\end{equation}
where $\bar\alpha_i$ is the average of the dual weights of class $i$ and
$\hat c_{i0}$ is an estimator of the noise floor.

\begin{itemize}
\item SVM: the hard-margin SVM without any correction.
\item BC-SVM($\kappa$): the correction~\eqref{eq:11.3} with
$\hat c_{i0}=\tr(S_i)/\hat\theta$, where $S_i$ is the sample covariance matrix.
This uses the full trace and corresponds to the existing BC-SVM.
\item BC-SVM($\kappa_*$): the correction~\eqref{eq:11.3} with
$\hat c_{i0}=\{\tr(S_i)-\sum_{s\le m_i}\tilde\lambda_{is}\}/\hat\theta$, where
$\tilde\lambda_{is}$ is the NR estimator of the spiked eigenvalue. This
corresponds to Proposition~\ref{prop:1}.
\item SC-SVM: the procedure given in Definition~\ref{def:2}. We give two versions.
In the version with the sample splitting, which is the one assumed in
Theorem~\ref{thm:5}, we use one half of the observations for the NR estimation and
the other half for the training. In the practical version, we use all the
observations for both steps.
\end{itemize}

We assumed that the number of the spikes is known and equal to one. We took
$d=20000$, $n_1=12$, $n_2=6$, $\delta=0.05$ and $100$ replications. We note that
the sample sizes are imbalanced, so that the bias correction is meaningful. The
implementation is given in \texttt{spiked\_svm\_classifiers.py}. We
give the results in Table~\ref{tab:compare}.

\begin{table}[ht]
\centering
\caption{The average of $e(1)+e(2)$ for $d=20{,}000$, $n_1=12$, $n_2=6$
and $\delta=0.05$. Standard errors in parentheses.}
\label{tab:compare}
\setlength{\tabcolsep}{3pt}
\begin{tabular}{@{}lrrrrr@{}}
\toprule
 & SVM & BC-SVM($\kappa$) & BC-SVM($\kappa_*$) & SC-SVM & SC-SVM (split) \\
\midrule
non-spiked
  & 1.0000 (0.0000) & 0.0000 (0.0000) & 0.0000 (0.0000)
  & 0.0015 (0.0001) & 0.0000 (0.0000) \\
$c=0.2$
  & 0.7887 (0.0135) & 0.2504 (0.0255) & 0.1600 (0.0200)
  & 0.1973 (0.0206) & 0.3684 (0.0279) \\
$c=0.5$
  & 0.3137 (0.0191) & 0.4071 (0.0383) & 0.1062 (0.0175)
  & 0.1185 (0.0171) & 0.2858 (0.0273) \\
$c=0.8$
  & 0.0347 (0.0099) & 0.4946 (0.0454) & 0.0220 (0.0082)
  & 0.0170 (0.0055) & 0.0710 (0.0131) \\
\bottomrule
\end{tabular}
\end{table}

We observe the following. In the non-spiked case, the SVM gives
$e(1)+e(2)=1.0000$, which is the strong inconsistency caused by the imbalance of
the sample sizes, and the bias correction removes it completely. This reproduces
the result of \citet{nakayama2017}. Under the spiked model, on the other hand,
BC-SVM($\kappa$) becomes worse as $c$ increases, and for $c=0.8$ it is much worse
than the SVM without any correction. This is the overcorrection given in
Proposition~\ref{prop:1}. In contrast, BC-SVM($\kappa_*$) gives a good performance
for all $c$.

In order to see the dependence on the sample size, we took $d=20000$, $c=0.8$ and
$\delta=0.01$, which is a harder setting, and varied $(n_1,n_2)$. We give the
results in Figure~\ref{fig:n-scale}.

\begin{figure}[ht]
\centering
\includegraphics[width=0.72\textwidth]{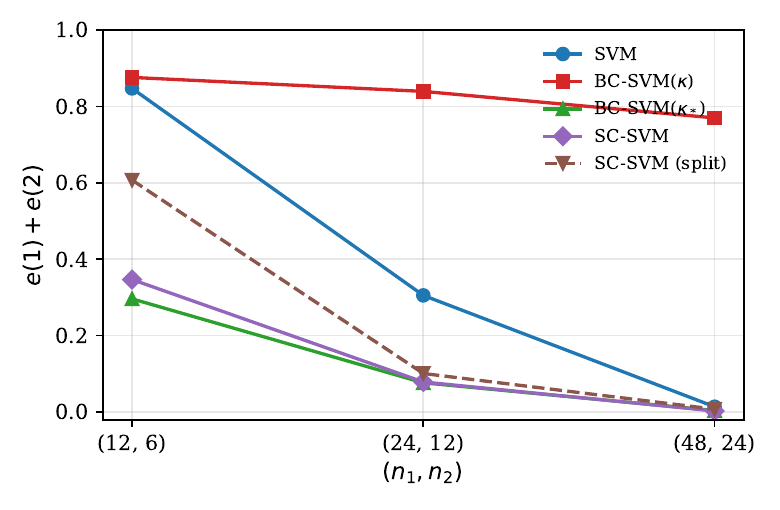}
\caption{Average of $e(1)+e(2)$ for $d=20{,}000$, $c=0.8$, $\delta=0.01$, as
$(n_1,n_2)$ grows.}
\label{fig:n-scale}
\end{figure}

We emphasize the behaviour of BC-SVM($\kappa$). Its misclassification rate is
$0.8759$, $0.8393$ and $0.7699$ at $(n_1,n_2)=(12,6)$, $(24,12)$ and $(48,24)$,
respectively, that is, it does not improve even though the sample size is
multiplied by four. This is because the overcorrection given in
Proposition~\ref{prop:1} is a systematic bias which does not vanish as the sample
size increases. On the other hand, BC-SVM($\kappa_*$) and the SC-SVM improve
rapidly. Hence Proposition~\ref{prop:1} is essential in practice.

We give the following remarks, which we consider to be informative for the users.

\begin{enumerate}[label=(\roman*)]
\item The SC-SVM and BC-SVM($\kappa_*$) give similar performances in
Table~\ref{tab:compare} and in Figure~\ref{fig:n-scale}, and the SC-SVM is better
only for $c=0.8$ in Table~\ref{tab:compare} and for the largest sample size in
Figure~\ref{fig:n-scale}. This is because the spike is orthogonal to $\mu$ in our
setting, so that the main damage caused by the spike comes through the bias term,
which BC-SVM($\kappa_*$) already removes. We expect that the advantage of the
SC-SVM is larger when the spike is not orthogonal to $\mu$ and the projection
removes a genuine source of noise in the discriminant direction.

\item The version with the sample splitting is worse than the practical version in
all the cases. This is because the splitting halves the number of the observations
used for the training, which is expensive when the sample size is small. We use
the splitting in Theorem~\ref{thm:5} in order to make the projected data independent
of the training data. We recommend the practical version in applications.

\item The gap between the two versions of the SC-SVM decreases as the sample size
increases, that is, it is $0.2603$, $0.0222$ and $0.0048$ at
$(n_1,n_2)=(12,6)$, $(24,12)$ and $(48,24)$, respectively. The gap is consistent
with Proposition~\ref{prop:3} in the sense that the estimation of the spiked
direction is poor when the sample size is small, so that the projection does not
work well.
\end{enumerate}

\section{Concluding remarks}
\label{sec:conclude}

In this paper, we considered asymptotic properties of the SVM in HDLSS settings
under a spiked model. We showed that the geometric representation of HDLSS data
does not hold under the spiked model, and that the HDLSS data converge to a random
configuration in a finite-dimensional space. We showed that the SVM does not hold
the consistency property, and that the bias term of the BC-SVM should be modified.
In order to overcome such difficulties, we proposed the SC-SVM and showed that it
holds the consistency property when the sample size goes to infinity. We also
showed that the growth of the sample size is essential for the projection-based
procedures. Finally, we verified Theorems~\ref{thm:1} and~\ref{thm:4},
Corollary~\ref{cor:1} and Proposition~\ref{prop:1} numerically, and we compared the
classifiers.

We give the following comparison between the existing theory and the results in
this paper.

\begin{table}[ht]
\centering
\caption{Comparison between the non-spiked model and the spiked model.}
\label{tab:model-compare}
\begin{tabular}{@{}p{3.2cm}p{5.2cm}p{5.2cm}@{}}
\toprule
Feature & non-spiked model & spiked model \\
\midrule
eigenvalue condition
& $\tr(\Sigma_i^2)/\theta_i^2\to 0$
& $\lambda_{is}/\theta\to c_{is}\in(0,1)$ \\
data configuration
& deterministic regular simplex
& random configuration in $\R^{1+m_1+m_2}$ \\
Gram matrix
& deterministic limit
& random $G^*$ \\
positive definiteness
& from the geometric representation
& only from the diagonal $D$ \\
SVM solution
& deterministic, closed form
& random limit \\
noise of $x_0$
& contributes as a diagonal term
& does not contribute \\
misclassification rate
& tends to zero by the BC-SVM
& does not tend to zero \\
bias term
& $\kappa$ in~\eqref{eq:1.3}
& $\kappa_*$ in Proposition~\ref{prop:1} \\
remedy
& bias correction
& spike removal and bias correction \\
required sample size
& $N$ may be fixed
& $n_i\to\infty$ is essential \\
\bottomrule
\end{tabular}
\end{table}

We close the paper with some remaining problems.

\begin{enumerate}[label=(\roman*)]
\item The theory of the SVM with the Gaussian kernel given by \citet{nakayama2020}
relies on the fact that the argument of the kernel function becomes deterministic
in the limit. From Theorem~\ref{thm:1}, the argument is random under the spiked
model, so that the kernel matrix itself has a random limit. It is a natural
problem to investigate how the choice of the scale parameter should be modified.

\item As we remarked in Remark~\ref{rem:7}, a minimax type lower bound for fixed
$N$ remains open.

\item The SC-SVM uses the NR methodology as a preprocessing. It is an interesting
problem to construct a framework in which the projection and the separating
hyperplane are estimated simultaneously. We also note from
Section~\ref{sec:sim-classifiers} that the sample splitting assumed in
Theorem~\ref{thm:5} is expensive in practice, so that a theory without the
splitting is desirable.

\item We assumed~(S), in which the spiked eigenvalues are of the same order as
$\theta$. On the other hand, there is an intermediate region in which
$\lambda_{i1}/\theta\to 0$ while $\lambda_{i1}^2/\tr(\Sigma_i^2)$ does not tend to
zero. In this region, the geometric representation is preserved but the second
order asymptotics is different, so that we expect that the consistency property is
preserved while the convergence rate is changed.

\item We considered the hard-margin SVM. When a regularization parameter $C$ is
introduced, the feasible set of the dual problem becomes
$\{\alpha:0\le\alpha_p\le C\}$ and the a priori bound in Lemma~\ref{lem:5} becomes
easier, so that the arguments in Section~\ref{sec:svm} carry over. On the other
hand, we expect that the limit depends on the relative scale of $C$ and $\theta$,
and the classification of the cases is not yet organized.

\item In the simulations we took the spiked direction orthogonal to $\mu$. As we
noted in Section~\ref{sec:sim-classifiers}, the advantage of the SC-SVM over the
bias correction alone is expected to be larger when the spike is not orthogonal to
$\mu$. A systematic study of this case, together with an application to actual
data, is left for future work.
\end{enumerate}

\appendix
\section{Mathematical tools}
\label{app:A}

In this appendix, we summarize the two tools used in Sections~\ref{sec:svm}
and~\ref{sec:disc}. Both of them are standard, but we give their intuitive meanings
and explain why they are necessary in this paper.

\subsection{The maximum theorem of Berge}
\label{app:A.1}

\begin{theorem}[Berge, 1963]
\label{thm:A.1}
Let $\Theta$ and $X$ be metric spaces, let $C:\Theta\rightrightarrows X$ be a
continuous correspondence with non-empty compact values, and let
$F:X\times\Theta\to\R$ be a continuous function. Put
\[
v(\vartheta)=\sup_{x\in C(\vartheta)}F(x,\vartheta),\qquad
M(\vartheta)=\arg\max_{x\in C(\vartheta)}F(x,\vartheta).
\]
Then $v$ is continuous and $M$ is upper hemicontinuous with non-empty compact
values. The upper hemicontinuity is written as
\begin{equation}
\vartheta_n\to\vartheta,\quad x_n\in M(\vartheta_n),\quad x_n\to x
\implies x\in M(\vartheta).
\label{eq:A.1}
\end{equation}
\end{theorem}

The content of the theorem is stated in one sentence as follows: the optimal value
moves continuously, while the optimal solution may jump, but the point to which it
jumps always belongs to the solution set of the limiting problem.

We give a minimal example. Let $X=\R$, $\Theta=\R$, $C(\vartheta)=[-1,1]$ and
$F(x,\vartheta)=\vartheta x$. Then $v(\vartheta)=|\vartheta|$ is continuous, while
\[
M(\vartheta)=\begin{cases}
\{1\} & \text{if }\vartheta>0,\\
[-1,1] & \text{if }\vartheta=0,\\
\{-1\} & \text{if }\vartheta<0,
\end{cases}
\]
so that the optimal solution jumps from $\{1\}$ to $\{-1\}$ at $\vartheta=0$.
Nevertheless~\eqref{eq:A.1} holds, because for $\vartheta_n\to 0$ we have
$x_n\in M(\vartheta_n)$ and indeed $x\in M(0)=[-1,1]$. That is, the upper
hemicontinuity states that the solution never appears at an irrelevant place, while
it allows the solution set to expand suddenly.

\begin{corollary}
\label{cor:A.1}
In the situation of Theorem~\ref{thm:A.1}, assume in addition that $M(\vartheta)$
is a singleton for each $\vartheta$. Then the map $\vartheta\mapsto M(\vartheta)$
is continuous.
\end{corollary}

\begin{proof}
Let $\vartheta_n\to\vartheta$ and $x_n=M(\vartheta_n)$. Take any subsequence of
$\{x_n\}$. From the compactness, it has a further subsequence converging to some
$x$. From~\eqref{eq:A.1} we have $x\in M(\vartheta)$, and from the uniqueness we
have $x=M(\vartheta)$. That is, every subsequence has a further subsequence
converging to the same limit $M(\vartheta)$, so that the whole sequence converges
to $M(\vartheta)$.
\end{proof}

We explain why this theorem is necessary. The solution of the SVM cannot be
written as an explicit function of $G$, because the set of the active constraints
changes with $G$. As long as the active set is fixed, the Karush--Kuhn--Tucker
conditions form a system of linear equations and one can obtain the
differentiability by the implicit function theorem. On the boundary at which the
active set switches, however, the differentiability is lost in general. What we
need in this paper is not the differentiability but only the continuity, and
Theorem~\ref{thm:A.1} gives it from the continuity of the objective function and
the feasible correspondence alone.

In order to apply the theorem, we have to supply the two conditions. As for the
continuity of the correspondence, we derive the a priori bound $\|\alpha\|\le K_0$
from the strict concavity and restrict the feasible set to a compact set which
does not depend on $G$. A constant correspondence is trivially continuous, so that
the most delicate condition of the theorem is met without any effort. As for the
uniqueness, we use Lemma~\ref{lem:4}.

\subsection{Fatou's lemma}
\label{app:A.2}

\begin{theorem}[Fatou's lemma]
\label{thm:A.2}
Let $\{f_n\}$ be a sequence of non-negative measurable functions on a measure
space $(X,\mathcal{F},\nu)$. Then
\begin{equation}
\int\liminf_{n\to\infty}f_n\,d\nu
\le\liminf_{n\to\infty}\int f_n\,d\nu.
\label{eq:A.2}
\end{equation}
\end{theorem}

In the language of probability, for non-negative random variables $\{X_n\}$ it
holds that $E\{\liminf X_n\}\le\liminf E\{X_n\}$.

The content of the lemma is stated as follows: a non-negative mass may escape in
the limiting procedure, but it never emerges.

We give the standard example in order not to mistake the direction of the
inequality. Let $X=[0,1]$ with the Lebesgue measure and $f_n=n\cdot I_{[0,1/n]}$.
For each $x>0$ we have $f_n(x)=0$ for sufficiently large $n$, so that
$\liminf f_n=0$ almost everywhere. On the other hand, $\int f_n=1$ for every $n$.
Hence~\eqref{eq:A.2} is a strict inequality. The mass of one escaped through the
pointwise convergence as a thin spike of height $n$ and width $1/n$. We note that
the non-negativity is necessary, because for $g_n=-f_n$ we have
$\int\liminf g_n=0$ and $\liminf\int g_n=-1$, so that~\eqref{eq:A.2} does not hold.

We explain why this lemma is necessary. In Corollary~\ref{cor:1}, what we have is
the convergence in distribution $e(1)\wto e^*(1)$ only, which does not imply the
convergence of the expectation in general. On the other hand, what we need is the
lower bound $E\{e^*(1)\}\le\liminf E\{e(1)\}$, which is exactly the direction of
the inequality in~\eqref{eq:A.2}. That is, in order to show the inconsistency, the
lower bound suffices and Fatou's lemma is enough.

Since~\eqref{eq:A.2} is a statement about the almost sure convergence, we use the
following bridge.

\begin{lemma}
\label{lem:7}
If $X_d\ge 0$ and $X_d\wto X$, then $E\{X\}\le\liminf_d E\{X_d\}$.
\end{lemma}

\begin{proof}
From the Skorokhod representation theorem \citep{billingsley1999}, there exist
random variables $Y_d$ and $Y$ on a common probability space such that $Y_d$ and
$X_d$ have the same distribution, $Y$ and $X$ have the same distribution, and
$Y_d\to Y$ almost surely. Applying Theorem~\ref{thm:A.2} to $\{Y_d\}$
\citep{durrett2019}, we obtain $E\{Y\}\le\liminf E\{Y_d\}$. Since the
expectation depends only on the distribution, we obtain the result.
\end{proof}

\begin{remark}
\label{rem:A.1}
In our case, we actually have the equality. Since $e(1)$ is a probability, it
holds that $0\le e(1)\le 1$, that is, the sequence is uniformly bounded. Applying
the Portmanteau theorem to the identity function on $[0,1]$, we obtain
$E\{e(1)\}\to E\{e^*(1)\}$. We use Fatou's lemma because the lower bound is what
we need and because it does not require the uniform boundedness. When one measures
the error rate in another scale which is not bounded, the lower bound given by
Fatou's lemma becomes essential.
\end{remark}

\section*{Acknowledgments}
The author thanks anonymous readers for helpful comments.
Parts of the drafting and technical editing of this manuscript were assisted by
large language models, including Anthropic's Claude Opus~5.
The author takes full responsibility for the correctness of all mathematical
statements, proofs, and numerical results.

\end{document}